%% file: aistats.tex
\documentclass[twoside]{article}
\usepackage[accepted]{aistats2023}
\usepackage[round]{natbib}

\usepackage{amsthm}

\input{math_commands.tex}

\usepackage{adjustbox}
\usepackage{lipsum}
\usepackage{hyperref}
\usepackage{url}

\newcommand{\indep}{\perp \!\!\! \perp}
\usepackage[ruled,vlined]{algorithm2e}
\usepackage{xcolor}
\usepackage{subfig}
\usepackage{graphicx}
\usepackage{float}

\begin{document}

%
\runningtitle{Information-Theoretic Bounds for Multi-Source DA}

%

\twocolumn[

\aistatstitle{Algorithm-Dependent Bounds for Representation Learning\\ of Multi-Source Domain Adaptation}

\aistatsauthor{ Qi Chen \And Mario Marchand }

\aistatsaddress{ Université Laval \And  Université Laval } ]

\begin{abstract}
We use information-theoretic tools to derive a novel analysis of Multi-source Domain Adaptation (MDA) from the representation learning perspective. Concretely, we study joint distribution alignment for supervised MDA with few target labels and unsupervised MDA with pseudo labels, where the latter is relatively hard and less commonly studied. We further provide algorithm-dependent generalization bounds for these two settings, where the generalization is characterized by the mutual information between the parameters and the data. Then we propose a novel deep MDA algorithm, implicitly addressing the target shift through joint alignment. Finally, the mutual information bounds are extended to this algorithm providing a non-vacuous gradient-norm estimation. The proposed algorithm has comparable performance to the state-of-the-art on target-shifted MDA benchmark with improved memory efficiency.

\end{abstract}

\section{INTRODUCTION}

The usual machine learning theories assume the test data follows the same distribution as the train data, which is often violated in real-world applications. Such a distribution (domain) shift degrades the algorithm's performance. So, various methods have been proposed to address this problem through Domain Adaptation (DA) or transfer learning
\citep{huang2007correcting,ben2010theory,pan2009survey}. Domain adaptation aims to learn a well-generalized predictor for the target domain $\Tcal$ with data sampled from the source domain $\Scal$, where $\Scal\neq\Tcal$. As deep learning becomes increasingly popular with its superior performance in many complex tasks such as computer vision \citep{csurka2017comprehensive} and natural language processing \citep{blitzer2008domain}, a series of works on deep domain adaption \citep{ganin2016domain,tzeng2017adversarial} have gained tremendous success in practice. 

The works mentioned above are focused on single-source domain adaptation. However, leveraging knowledge from multiple sources ($\Scal_{1:N}$) is more attractive in practice. Simply merging all the sources as one single source and then applying the single-source domain adaptation algorithms is obviously suboptimal since some source distributions will generally be more similar to the target than others. In that case, a uniformly merged source can be extremely dissimilar to the target distribution and violate the conditions for successful domain adaptation, e.g., those proposed by \citet{ben2010impossibility}. 

Hence, to learn from most related sources, many existing works in MDA \citep{zhao2018adversarial,wen2020domain,peng2019moment,li2018extracting} use the divergences between each source marginal and the target marginal $d(\Scal_i(X),\Tcal(X)), \forall i \in [N]$ for inferring the domain relations. Such marginal alignment approaches have been demonstrated to be problematic even in single-source DA \citep{zhao2019learning} when the label distributions are different $\Tcal(Y)\neq\Scal(Y)$ --- which we denote as the \textbf{target shift}\footnote{We use the terminology of \textbf{target shift} in the rest of the paper to avoid confusion with the label shift assumption, where $\Scal(X|Y)=\Tcal(X|Y), \Scal(Y)\neq\Tcal(Y)$.}.
Recently, other approaches \citep{redko2019optimal, shui2021aggregating} have been proposed with an additional algorithmic layer that estimates the label ratio to mitigate the target shift. However, the estimation may be inaccurate in some situations, consequently degrading the performance.

The aforementioned target shift correcting methods are mathematically a special case of the joint distribution alignment. Moreover, when the number $N$ of source domains is very large, the previous pairwise alignment methods will suffer from estimating numerous discriminators w.r.t both memory and sample complexity. Therefore, we propose to learn a joint alignment between a convex combination of source distributions  $\Scal^{\vect{\alpha}}=\sum_{i=1}^N\alpha_i \Scal_i, \forall \vect{\alpha} \in \Delta_N$ and the target distribution $\Tcal$, where the vector $\vect{\alpha}$ of domain weights is optimized by the learning algorithm. The alignments are performed on a hidden-layer representation space commonly used for deep learning, while the corresponding analysis is rare. Moreover, the previous theories using Rademacher complexity or VC dimension are not algorithm-dependent. In contrast, we use information-theoretic tools to conduct the algorithm-specific generalization error analysis for representation learning of MDA. 

Compared with previous works, the highlighted contributions of this paper are as follows:




\paragraph{Unified Approach} We propose a unified approach for supervised and unsupervised MDA that conducts joint alignment w.r.t the representation space, which is appropriate for the target shift problem. The proposed algorithm simultaneously learns the domain weights, and the proposed non-pairwise alignment is more memory efficient than the previous works. Moreover, the proposed algorithm outperforms the previous approaches under a significant target shift in the unsupervised scenario, which is relevant in practice. 

\paragraph{Algorithm-Dependent Generalization Bounds} We first provide algorithm-dependent generalization bounds for MDA that depend on the mutual information between the model parameters and the input data: the less the model parameters depend on the data, the less the algorithm will overfit. We further apply the bounds on the proposed deep MDA algorithm and obtain a gradient norm estimation, which is used as a regularization coefficient for optimizing the domain weights.


\section{RELATED WORK}
Since domain adaptation has plenty of works in various settings, we only discuss the most related approaches in this section.

\paragraph{Metrics for Distribution Shift in DA} Early theoretical works on single-source DA are based on the $\Hcal$-divergence \citep{ben2006analysis, ben2010theory, ben2010impossibility}. Then there emerge extensions on exploring different metrics, \eg, discrepancy \citep{mansour2009domain}, Wasserstein distance \citep{courty2014domain, courty2017joint}, Jensen-Shannon divergence \citep{shui2020beyond, shui2022novel} and more general definitions of divergence like Integral Probability Metrics (IPM) \citep{zhang2012generalization} and $f$-divergence \citep{acuna2021f}. For more comprehensive discussions, please address to
\citet{redko2019advances} and \citet{wang2023gap}. 

In this paper, we adopt Wasserstein distance for its numerous nice properties. It's tighter than the KL divergence under the sub-Gaussian assumption through the transportation cost inequality and can capture the underlying geometry of the data. \citet{courty2017joint} first theoretically analyzes joint alignment on \textit{example space} in single-source DA using Wasserstein distance. The algorithm based on solving the optimal transport problem scales quadratically in sample size, which is \emph{intractable on large datasets}. Hence, \cite{damodaran2018deepjdot} empirically addresses this difficulty by alternately updating the mini-batch coupling matrix and network parameters to approximate the Wasserstein on the representation space. We theoretically analyze the joint alignment on the representation space w.r.t Wasserstein distance, providing algorithm-dependent bounds.

\paragraph{Information-Theoretic Learning in DA} Recently, the information-theoretic analysis introduced by \citet{xu2017information} and \citet{russo2019much} has been used to provide a rigorous understanding of the generalization capabilities of deep learning models, such as complex meta-learning algorithms \citep{chen2021generalization, jose2020information}.  In contrast with the conventional VC-dimension and uniform stability bounds, it has the signiﬁcant advantage of incorporating the dependence on the data distribution, the hypothesis space, and the learning algorithm. \cite{wu2020information, wu2022information} also use information-theoretic tools to derive bounds for single-source supervised DA. However, their bounds contain a non-optimizable term $\KL(\Tcal,\Scal)$. On the contrary, we first provide fully algorithm-dependent bounds for both the unsupervised and supervised MDA, without loss of generality for single-source DA.

\paragraph{Multi-Source DA} \citet{konstantinov2019robust} and \citet{shui2021aggregating} consider the \emph{supervised} MDA case when few target labels are available. \cite{konstantinov2019robust} estimates the domain relations with the pair-wise discrepancy $d(\Scal_i, \Tcal)$. \cite{shui2021aggregating} reweights the losses and Wasserstein distances between domains with label ratio to mitigate the target shift, where the label ratio is estimated with the statistics given the few target labels and the source labels. 

For \emph{unsupervised} MDA, most existing approaches can be divided into two categories w.r.t the target shift. Marginal alignment methods make use of different divergences between each source marginal distribution and the target marginal distribution (\ie, $d(\Scal_i(X),\Tcal(X)), \forall i \in [N]$) to estimate the domain relations. E.g., \citet{zhao2018adversarial} and \citet{li2018extracting} use the $\Hcal$-divergence, which is only suitable for binary classification. \citet{wen2020domain} and \citet{peng2019moment} adopt the discrepancy. The above methods that conduct unsupervised marginal alignments are unstable w.r.t target shift \citep{zhao2019learning}.

\citet{redko2019optimal} and \citet{shui2021aggregating} consider addressing target shift with special cases of the joint distribution alignment. JCPOT \citep{redko2019optimal} assumes $\Scal_i(X|Y)=\Tcal(X|Y), \forall i \in [N]$ and uses iterative Bregman projection to solve the constrained Wasserstein barycenter problem, optimizing the class proportion to correct the target shift. However, the above label shift assumption often does not hold, and similar to \citet{courty2017joint}, the algorithm is \emph{intractable on large datasets}. \citet{shui2021aggregating} assumes the Generalized Label Shift (GLS) condition \citep{combes2020domain} is satisfied and estimates the label ratio with a black box shift estimation \citep{lipton2018detecting}. In addition, they minimize $d(\Scal_i(X|Y), \Tcal(X|\hat{Y}))$ on representation space using the Wasserstein distance and the predicted pseudo labels. The theorems in \cite{shui2021aggregating} are proposed for supervised MDA, without formal justification for unsupervised MDA using pseudo labels. MOST \citep{nguyen2021most} does not optimize the domain relations and simply learns a source domain discriminator to obtain a \emph{weighted ensemble expert} (teacher). A student classifier imitates the teacher by minimizing the Wasserstein distance between the \emph{pseudo} source and \emph{pseudo} target joint distribution using predicted labels. 

Our approach differs from all the above-mentioned methods. We do not conduct a pairwise alignment. We define a combined source distribution and directly minimize the Wasserstein distance between the combined source distribution and the target distribution. We do not explicitly estimate the label ratio, and the target shift is tackled implicitly with the joint representation alignment. Moreover, we optimize the task weights $\alpha$ w.r.t the target and source domain shifts with an adaptive information-theoretical regularization coefficient.

\section{PROBLEM SETUP}
\paragraph{Basic Notations} Without specification, we use upper case letters to denote random variables and the corresponding calligraphic letters to denote the corresponding sets on which they are defined, \eg, $X,Y$ on $\Xcal,\Ycal$. 

Let $\Xcal$ be the \textit{input space}. Let $\Ycal$ be the \textit{label space} and let $\Zcal\eqdef \Xcal\times\Ycal$ be the \textit{example space}. Then a set of $N$ source distributions $\Scal_{1:N} \eqdef \{\Scal_1,\ldots,\Scal_N\}$ and the target distribution $\Tcal$ are defined on $\Zcal$. 
Now we define a representation learning function $g:\Ucal \times \Xcal \rightarrow \tilde{\Xcal}$ that transforms the inputs to feature representations, where $g$ is parameterized by $u$ defined on $\Ucal$. Consequently, the distribution on the original example space has an induced distribution on the new space $\tilde{\Zcal} \eqdef \tilde{\Xcal} \times \Ycal$. Given the weight simplex $\Delta_N\eqdef\{\vect{\alpha}:\alpha_i\geq 0, \sum_{i=1}^N \alpha_i = 1\}$, we consider a mixture of source distributions $\Scal^{\vect{\alpha}} \eqdef \sum_{i=1}^N \alpha_i \Scal_i$ as the \textit{combined source distribution}. The induced distributions on $\tilde{\Zcal}$ are determined by $u$ and denoted as $\tilde{\Tcal}_u$ and $\tilde{\Scal}^{\alpha}_u  \eqdef \sum_{i=1}^N \alpha_i \tilde{\Scal}_{i,u}$.


Then given the parameterized hypothesis space $\Vcal$, we define a predictor $h:\Vcal \times \tilde{\Xcal} \rightarrow \Ycal$ that predicts a label given the feature representation. To measure the performance of the predictor, let $\ell: \Ycal \times \Ycal \rightarrow \Reals^+$ be a positive-valued loss function. 

Hence, the \textbf{population(true) risk} of $(u,v)\in \Ucal\times\Vcal$ on the target distribution and the combined source distribution are respectively defined as 
\[
\begin{aligned}
R_{\Tcal}(u,v) &\eqdef  \EE_{Z\sim\Tcal}\ell(h(v, g(u,X)), Y)\\
&= \EE_{\tilde{Z} \sim \tilde{\Tcal}_u} \ell(h(v, \tilde{X}), Y)\\
R_{\Scal^{\vect{\alpha}}}(u, v) &\eqdef \sum_{i=1}^N \alpha_i \EE_{Z\sim \Scal_i}\ell(h(v, g(u,X)), Y) \\
&= \EE_{\tilde{Z}\sim \tilde{\Scal}_u^{\vect{\alpha}}} \ell(h(v, \tilde{X}), Y)\,,
\end{aligned}
\]
where $Z\eqdef(X,Y), \tilde{Z}\eqdef(\tilde{X}, Y)$.

Let the source datasets $S_{1:N}$ and the target dataset $T$ be the examples sampled from the corresponding distributions, where $S_i = \{Z_{i,j}^s\}_{j=1}^{m_i}, Z_{i,j}^s = (X_{i,j}^s, Y_{i,j}^s) \sim \Scal_i, \forall i \in [N]$ and $T = \{Z_{j}^t\}_{j=1}^{m_t}, Z_j^t = (X_j^t, Y_j^t)\sim \Tcal$. We further denote the corresponding dataset for the combined source distribution as $S^{\vect{\alpha}}$. Consequently, the corresponding \textbf{empirical risks} are defined as:

\hspace{0.5cm}$
\hat{R}_{\Tcal}(u,v) \eqdef \frac{1}{m_t}\sum_{j=1}^{m_t}\ell(h(v,g(u,X_j^t)),Y_j^t)$

\hspace{0.5cm}$\hat{R}_{\Scal^{\vect{\alpha}}}(u, v) \eqdef \sum_{i=1}^N \frac{\alpha_i}{m_i} \sum_{j=1}^{m_i}\ell(h(v,g(u,X_{i,j}^s)),Y_{i,j}^s)\, .
$

The target labels are not accessible during the training phase for \textbf{unsupervised} DA. Thus, several methods  \citep{xie2018learning,courty2017joint} apply the predicted labels of target inputs as \textbf{pseudo labels}. We formulate the corresponding setting in MDA. Let $f_p:\Xcal\rightarrow\Ycal$ be a pseudo labeling function that maps the target inputs to the label space. In this paper, we set $f_p(x)=h(v,g(u,x))$ focusing on representation learning. Therefore, $\forall u,v \in \Ucal \times \Vcal$, let $\Tcal_{u,v}$ be the \textit{pseudo target distribution} on $\Zcal=\Xcal \times \Ycal$ induced by $f_p$, and $\tilde{\Tcal}_{u,v}$ be the corresponding distribution on $\tilde{\Zcal}=\tilde{\Xcal}\times\Ycal$.
So we have a random sample from the pseudo distribution noted as $\hat{Z} \eqdef (X, \hat{Y}) = (X, f_p(X))$. Finally, let us define the unlabeled target dataset as $T_X^\prime= \{X_j^t\}_{j=1}^{m_t^\prime}, X_j^t\sim \Tcal(X)$, where $\Tcal(X)$ denotes the target distribution $\Tcal$ marginalized on $\Xcal$. 


Now we introduce some common assumptions for deriving the multi-source DA bounds in this paper.
\begin{assumption}\label{gen_assump}
\textbf{(a)} The representation learning function $g:\Ucal\times\Xcal\rightarrow\tilde{\Xcal}$ is $K$-Lipschitz \footnote{The general definition of Lipschitzness is provided in the Appendix~\ref{def:lip}.} over $\Xcal$ for any $u \in \Ucal$ w.r.t metrics $\rho_{\tilde{x}}$ and $\rho_{x}$, \ie, $\rho_{\tilde{x}}(g(u,x), g(u,x^\prime)) \leq K \rho_{x}(x, x^\prime)$. \textbf{(b)} The predictor $h:\Vcal \times\tilde{\Xcal}\rightarrow\Ycal$ is $L$-Lipschitz  over $\tilde{\Xcal}$ for any $v \in \Vcal$ w.r.t metrics $\rho_{y}$ and $\rho_{\tilde{x}}$, \ie, $\rho_y(h(v, \tilde{x}), h(v, \tilde{x}^\prime)) \leq L \rho_{\tilde{x}}(\tilde{x}, \tilde{x}^\prime)$. \textbf{(c)} The loss function $\ell:\Ycal\times\Ycal \rightarrow \Reals^+$ is assumed to be symmetric, $M$-Lipschitz w.r.t $\rho_y$ in the first argument and satisfying the triangle inequality.
\end{assumption}


\section{SUPERVISED MDA}

We first consider the supervised MDA regime. In most case, only a \emph{few} amounts of labeled target data is accessible.

\subsection{Population Risk Bound}
Based on the optimal transport theory \citep{villani2009optimal,peyre2017computational}, we obtain the following population risk bound using the Wasserstein distance (Definition~\ref{def:def_wasserstein}).
\begin{theorem}\label{thm:sup_wasserstein}
$\forall (u,v) \in \Ucal\times\Vcal$, and $\forall \vect{\alpha} \in \Delta_N$, if Assumption~\ref{gen_assump} is satisfied, then 
\[
|R_{\Tcal}(u,v) - R_{\Scal^{\vect{\alpha}}}(u,v)| \leq  \mathbf{W}_1(\tilde{\Tcal}_u, \tilde{\Scal}^{\vect{\alpha}}_u) \leq \mathbf{W}_1(\Tcal, \Scal^{\vect{\alpha}})\,,
\]
where the first $\mathbf{W}_1$ distance is defined on the metric space $(\tilde{\Zcal}, \rho_{\tilde{z}})$, for $\rho_{\tilde{z}}(\tilde{z}, \tilde{z}^\prime) = \ell(y,y^\prime) + LM \rho_{\tilde{x}}(\tilde{x}, \tilde{x}^\prime)$, and the second one is defined on metric space $(\Zcal, \rho_z)$, for $\rho_{z}(z, z^\prime)= \ell(y,y^\prime) + LMK \rho_{x}(x, x^\prime)$.
\end{theorem}

Theorem~\ref{thm:sup_wasserstein} illustrates the importance of learning the transformation function $g$, where $ \mathbf{W}_1(\tilde{\Tcal}_u, \tilde{\Scal}^{\vect{\alpha}}_u)$ is always smaller than $\mathbf{W}_1(\Tcal, \Scal^{\vect{\alpha}})$. The latter is a fixed term that characterizes the shift between the combined source and the target distribution given $\vect{\alpha}$. More specifically, there may not exist an $\vect{\alpha}$ that gives $\mathbf{W}_1(\Tcal, \Scal^{\vect{\alpha}}) = 0$. On the contrary, by optimizing the transformation parameter $u$, we may obtain $\mathbf{W}_1(\tilde{\Tcal}_u, \tilde{\Scal}^{\vect{\alpha}}_u) = 0$. Detailed proof is provided in Appendix~\ref{proof:sup_pop_bound}.


\subsection{Generalization Bound}

From Theorem~\ref{thm:sup_wasserstein}, the target population  risk is bounded by $ R_{\Scal^{\vect{\alpha}}}(u, v) + \mathbf{W}_1(\tilde{\Tcal}_u, \tilde{\Scal}^{\vect{\alpha}}_u)$. The estimation of the Wasserstein distance requires target data. Considering the extreme case $m_t \rightarrow +\infty$, minimizing the empirical estimation of this upper bound will never be better than directly minimizing the target empirical risk.  So we propose to learn by considering the following \textbf{combined risk} with parameter $0\leq\epsilon\leq1$: 
\[
\begin{aligned}
R^{\epsilon}_{\Scal^{\vect{\alpha}}, \Tcal}(u, v) \eqdef (1-&\epsilon) R_{\Tcal}(u,v) \\
&+ \epsilon(R_{\Scal^{\vect{\alpha}}}(u, v) + \mathbf{W}_1(\tilde{\Tcal}_u, \tilde{\Scal}^{\vect{\alpha}}_u))\, .
\end{aligned}
\]
Thus, we have $R_{\Tcal}(u,v) \leq R^{\epsilon}_{\Scal^{\vect{\alpha}}, \Tcal}(u, v)$. Then, let us denote the empirical combined risk as 
$\hat{R}^{\epsilon}_{\Scal^{\vect{\alpha}},\Tcal}(u,v)=(1-\epsilon) \hat{R}_{\Tcal}(u,v)
+ \epsilon\hat{R}_{\Scal^{\vect{\alpha}}}(u, v) + \epsilon \hat{\mathbf{W}}_1(\tilde{\Tcal}_u, \tilde{\Scal}^{\vect{\alpha}}_u)$.


We consider a stochastic learning algorithm $\Acal$ for supervised MDA, which takes the $N$ source datasets and the target dataset as input and then outputs the random parameters $ (U,V) = \Acal(S^{\vect{\alpha}}, T) \sim P_{U,V|S^{\vect{\alpha}},T}$, where $P_{U,V|S^{\vect{\alpha}},T}$ is the distribution on $\Ucal\times\Vcal$ induced by $\Acal$, given $(S^\alpha, T)$. 
The expected generalization gap is then defined as
\[
\begin{aligned}
gen(\Tcal, \Acal) &\eqdef  \EE_{U,V, S^{\vect{\alpha}}, T} [R_\Tcal(U,V) - \hat{R}^{\epsilon}_{\Scal^{\vect{\alpha}}, \Tcal}(U, V)]\\
&\leq \EE_{U,V, S^{\vect{\alpha}}, T} [R^{\epsilon}_{\Scal^{\vect{\alpha}}, \Tcal}(U, V) - \hat{R}^{\epsilon}_{\Scal^{\vect{\alpha}}, \Tcal}(U, V)]\, .
\end{aligned}
\]
To bound the generalization gap for the Wasserstein distance, traditional methods apply the triangle inequality of the Wasserstein distance and then separately bound the gap for the source domain and target domain \citep{courty2017joint,shui2021aggregating}. The proofs are based on previous concentration results \citep{weed2019sharp} of the Wasserstein distance with the additional assumption of a bounded loss. To obtain a non-vacuous algorithm-dependent bound, we address the problem in a different way. According to the Kantorovich-Rubinstein duality \citep{villani2009optimal}, we have 
$$
\mathbf{W}_1(\tilde{\Tcal}_u, \tilde{\Scal}^{\vect{\alpha}}_u) 
=  \sup_{f: \|f\|_{Lip}\leq 1} \EE_{\tilde{z}\sim \tilde{\Tcal}_u} f(\tilde{z}) - \EE_{\tilde{z}^\prime \sim \tilde{\Scal}^{\vect{\alpha}}_u} f(\tilde{z}^\prime)\, .
$$
The supremum is over all the 1-Lipschitz functions $\Fcal\eqdef\{f:\tilde{\Zcal}\rightarrow \Reals^+$, $|f(\tilde{z}) - f(\tilde{z}^\prime)| \leq \rho_{\tilde{z}}(\tilde{z}, \tilde{z}^\prime)\}$. Because of the intractability of learning this class of functions \citep{arjovsky2017wasserstein}, let us consider instead a parameterized class of functions $\tilde{\Fcal}\eqdef\{\tilde{f}:\tilde{\Zcal}\times\Vcal^\prime \rightarrow \Reals^+,\|\tilde{f}\|_{Lip}\leq 1\}$ with the parameter set $\Vcal^\prime$. We assume the above supremum is attained in $\tilde{\Fcal}\subset\Fcal$. Under this assumption, we have 
$$\mathbf{W}_1(\tilde{\Tcal}_u, \tilde{\Scal}_u^{\vect{\alpha}}) = \sup_{v^\prime\in \Vcal^\prime} \EE_{\tilde{z}\sim \tilde{\Tcal}_u} \tilde{f}(v^\prime,\tilde{z}) - \EE_{\tilde{z}^\prime \sim \tilde{\Scal}^{\vect{\alpha}}_u} \tilde{f}(v^\prime,\tilde{z}^\prime)\,.$$
Now we bring up the sub-Gaussian (Definition \ref{def:subgaussian}) assumptions and provide a mutual information (Definition~\ref{def:mi}) bound for supervised MDA.
\begin{assumption}\label{assum_gauss} $\forall (u,v) \in \Ucal\times\Vcal$, $\ell(h(v, g(u,X)), Y)$ is $\sigma$-sub-Gaussian w.r.t $Z \sim \Tcal$ and $Z\sim \Scal_i, \forall i \in [N]$. And for any proper choice of $\tilde{\Fcal}$ with the supremum attained,
$\forall (u, v^\prime) \in \Ucal\times\Vcal^\prime$, $\tilde{f}(v^\prime, (g(u,X), Y)) \in \tilde{\Fcal}$ is $\sigma^\prime$-sub-Gaussian w.r.t $Z \sim \Tcal$ and $Z\sim \Scal_i, \forall i \in [N]$.
\end{assumption} 

\begin{theorem}\label{gen_supervised}
If Assumption~\ref{gen_assump} and \ref{assum_gauss} are satisfied, then we can bound the generalization gap of the supervised multi-source domain adaptation algorithm $\Acal$ for any $\vect{\alpha}\in \Delta_N$ and $0 \leq \epsilon \leq 1$ with: 
\[
\begin{aligned}
gen(\Tcal, \Acal) &\leq \sigma\sqrt{2(\frac{(1-\epsilon)^2}{m_t} + \sum_{i=1}^N\frac{\epsilon^2\alpha_i^2}{m_i}) I(U,V;S^{\vect{\alpha}}, T)} \\
& + \sigma^\prime\sqrt{2\epsilon^2(\sum_{i=1}^N\frac{\alpha_i^2}{m_i} + \frac{1}{m_t})I(U;S^{\vect{\alpha}},T)}\, .
\end{aligned}
\]
\end{theorem}
The above bound on the \textbf{generalization gap} for supervised MDA consists of two terms. The first term is characterized by the mutual information between the input datasets $S^{\vect{\alpha}}, T$ and the output hypothesis $U,V$. The second term, which comes from the estimation of the Wasserstein distance on the representation space, depends on the mutual information between the representation parameter $U$ and the input datasets. Notice that if we consider an arbitrarily random representation learning function (i.e., $U$ does not depend on any data), the mutual information terms are equal to zero. However, the empirical risk will be very large. See detailed proof of Theorem~\ref{gen_supervised} in the Appendix~\ref{proof:sup_gen_bound}. 

\paragraph{Exploit Target Data} Let us first focus on the case with few target data, \ie, $m_t$ is small. Usually, the number of source data samples is much more significant than $m_t$. So when $\sum_{i=1}^N \alpha_i^2/m_i \ll 1/m_t$, the generalization gap is dominated by the target data with rate $\Ocal(\sqrt{1/m_t})$.
In this case, it's preferable to use the target data for approximating the Wasserstein distance ($\epsilon=1$) compared to directly train on the target data, because $I(U;\Scal^{\vect{\alpha}}, T) \leq I(U,V;\Scal^{\vect{\alpha}}, T)$ assuming $\sigma=\sigma^\prime$. On the contrary, when $1/m_t \ll \sum_{i=1}^N \alpha_i^2/m_i$, we prefer to directly train on the target data ($\epsilon=0$). For other cases, there exists a non-trivial value for $\epsilon$ that minimizes the generalization gap w.r.t the algorithm.

\section{UNSUPERVISED MDA}

For unsupervised domain adaptation, the target labels are not accessible during the training phase. So it's hard to guarantee a successful transfer without any precondition.
\citet{ben2010theory} identified a necessary condition for successful unsupervised domain adaptation: the ideal joint error must be small enough, otherwise, no classifier can perform well on both the source and target domain. Therefore, a small ideal joint error is commonly assumed in the literature. Based on the same assumption, we provide the theorems for unsupervised MDA with pseudo labels.
\begin{definition}(Ideal joint error) Let the ideal joint hypothesis be $(u^*, v^*) \eqdef \argmin_{u,v} R_{\Scal^{\vect{\alpha}}}(u,v) + R_{\Tcal}(u,v)$, then the ideal joint error is $R^* = R_{\Scal^{\vect{\alpha}}}(u^*,v^*) + R_{\Tcal}(u^*,v^*)$. 
\end{definition}

\subsection{Population Risk Bound}
\begin{definition}(Joint Approximation Error) Let $\tilde{x}=g(u,x), \tilde{x}^\prime=g(u,x^\prime)$ and the \textbf{optimal} coupling on $\Zcal=\Xcal\times\Ycal$ given the pseudo labeling function $h(v,g(u,x))$ be 
\[
\gamma_{u,v}^* \eqdef \argmin\limits_{\gamma \in \Pi(\Tcal_{u,v}, \Scal^{\vect{\alpha}})}\int \left[ LM\rho_{\tilde{x}}(\tilde{x}, \tilde{x}^\prime) + \ell(\hat{y}, y^\prime) \right] d\gamma(\hat{z},z^\prime)
\]
Then the joint approximation error is
\[
\begin{aligned}
R^*_{rep}(u,v) &\eqdef LM \int_{\Zcal\times\Zcal} \left[ \rho_{\tilde{x}}(g(u^*,x) ,g(u^*,x^\prime))\right. \\
&\quad\quad\quad\quad \left. -\rho_{\tilde{x}}(g(u,x), g(u,x^\prime))\right] d\gamma_{u,v}^*(\hat{z}, z^\prime)\,,
\end{aligned}
\]
where $R_{rep}^*(u,v)$ characterizes the approximation error for the domain shift on the representation space of $u^*$ (the representation counterpart of the ideal joint hypothesis). 
\end{definition}


\begin{theorem}\label{thm:pseudo_label}
For any given $\vect{\alpha}\in\Delta_N$ and $\forall u,v \in \Ucal \times \Vcal$, if Assumption~\ref{gen_assump} is satisfied, 
then the target population risk can be bounded by: 
\[
R_{\Tcal}(u,v) \leq \mathbf{W}_1(\tilde{\Tcal}_{u,v}, \tilde{\Scal}_u^{\vect{\alpha}}) + R^*_{rep}(u,v) + R^*\,.
\]
\end{theorem}

See the detailed proof in Appendix~\ref{proof:unsup_pop_bound}.
From the above theorem, we notice that if we minimize $\mathbf{W}_1(\tilde{\Tcal}_{u,v},\tilde{\Scal}_u^{\vect{\alpha}})$, $R_{rep}^*(u,v)$ may increase. 
Hence, for a successful unsupervised MDA w.r.t joint representation alignment, the proposed theory implies that it is sufficient to have a small $R_{rep}^*(u,v)$ for the $(u,v)$ that minimizes $\mathbf{W}_1(\tilde{\Tcal}_{u,v}, \tilde{\Scal}_u^{\vect{\alpha}})$. 

\paragraph{Target Shift} An important benefit of the joint alignment methods is that it's favorable compared to the marginal alignment method when the target shift exists. Because with the joint optimal transport, we can automatically correct the target shift by minimizing $\mathbf{W}_1(\tilde{\Tcal}_{u,v}, \tilde{\Scal}_u^{\vect{\alpha}})$ w.r.t $v$. While the traditional marginal alignment methods that minimize $\mathbf{W}_1(\tilde{\Scal}_u(X), \tilde{\Tcal}_{u}(X)) + R_{\Scal^{\vect{\alpha}}}(u,v)$ cannot correct this shift since neither of the two terms in the objective are related to the divergence between the source and target label distributions. 



\subsection{Generalization Bound}
Similarly to the above supervised MDA method, we use a stochastic algorithm $\Acal^{un}$, which takes the $N$ source datasets and the unlabeled target dataset as input and then outputs the random parameters $ (U,V) = \Acal^{un}(S^{\vect{\alpha}}, T_X^\prime) \sim P_{U,V|S^{\vect{\alpha}},T_X^\prime}$. Combine with Theorem~\ref{thm:pseudo_label} , the expected generalization gap for unsupervised MDA is defined as:
$$gen(\Tcal,\Acal^{un}) \eqdef \EE_{U,V,S^{\vect{\alpha}},T_X^\prime} [R_{\Tcal}(U,V) - \hat{\mathbf{W}}_1(\tilde{\Tcal}_{U,V}, \tilde{\Scal}_U^{\vect{\alpha}})]\,,$$
where $\hat{\mathbf{W}}_1(\tilde{\Tcal}_{U,V}, \tilde{\Scal}_U^{\vect{\alpha}})$ is the empirical estimation of $\mathbf{W}_1(\tilde{\Tcal}_{U,V}, \tilde{\Scal}_U^{\vect{\alpha}})$ that has different forms w.r.t the specific choice of $\tilde{\Fcal}$, as discussed in Sec.~\ref{choice_f}. The definition without specification of $\tilde{\Fcal}$ is provided in the Appendix. Now, consider the following sub-Gaussian assumption. 
\begin{assumption}\label{assum_unsup}
For any proper choice of $\tilde{\Fcal}$ with the supremum attained, $\forall (u,v,v^\prime) \in \Ucal\times\Vcal\times\Vcal^\prime$, $\tilde{f}(v^\prime, (g(u,X), \hat{Y})) \in \tilde{\Fcal}$ is $\sigma^\prime$-sub-Gaussian w.r.t $\hat{Z} \sim \Tcal_{u,v}$ and $\forall (u, v^\prime) \in \Ucal\times\Vcal^\prime, \tilde{f}(v^\prime, (g(u,X), Y)) \in \tilde{\Fcal}$  is $\sigma^\prime$-sub-Gaussian w.r.t $Z\sim \Scal_i, \forall i \in [N]$. 
\end{assumption}
Then, we obtain the following algorithm-dependent bound on the generalization gap for joint representation alignment with pseudo labels.
\begin{theorem}\label{gen_un_pseudo}
If Assumption~\ref{gen_assump} and \ref{assum_unsup} are satisfied, then we have the following bound on the expected generalization gap of the unsupervised multi-source domain adaptation algorithm with pseudo label $\Acal^{un}$ for any $\vect{\alpha}\in\Delta_N$:
\[
\begin{aligned}
gen(\Tcal,\Acal^{un}) &\leq  \sqrt{2{\sigma^\prime}^2(\sum_{i=1}^N\frac{\alpha_i^2}{m_i} + \frac{1}{m_t^\prime})I(U,V;S^{\vect{\alpha}}, T_X^\prime)} \\
&+ \EE_{U,V} R^*_{rep}(U,V) + R^*
\end{aligned}
\]
\end{theorem}
\paragraph{Discussion} The above bound contains the two terms $\EE_{U,V}R^*_{rep}(U,V)$ and $R^*$, which are assumed to be small enough for successful unsupervised representation transfer. Then the generalization gap of the joint representation alignment algorithm is upper bounded by the mutual information between $(S^{\vect{\alpha}},T_X^\prime)$ and the algorithm output $(U,V)$. This is partly because the joint alignment also affects the learning of the predictor parameter $V$ through the pseudo-label predictions. For unsupervised MDA, we often have a large $m_t^\prime$, where $m_t^\prime \gg m_i, \forall i \in [N]$. So the sample complexity is dominated by $\sum_{i=1}^N\frac{\alpha_i^2}{m_i}$. Consider the special case where all the sources domains have equivalently strong similarity with the target domain such that we have $\alpha_i = 1/N$ and $\sum_{i=1}^N\frac{\alpha_i^2}{m_i} = \frac{\sum_{i=1}^N 1/m_i}{N^2}$. In that case, the generalization gap can be decreased by increasing the number of domains, \ie, $N\rightarrow \infty$. The detailed proof of Theorem~\ref{gen_un_pseudo} is provided in Appendix~\ref{proof:unsup_gen_bound}.

\section{UNIFIED APPROACH FOR DEEP MDA}\label{algo}

\subsection{Information-Theoretic MDA (IMDA)}
\paragraph{Combine Supervised and Unsupervised MDA}
The supervised MDA algorithm does not exploit the information conveyed in the unlabeled target data. At the same time, the unsupervised MDA algorithms require a small ideal joint error, which may not be satisfied in practice. Hence, we consider combining the two methods that both conduct joint representation alignment. Then we propose a general approach by balancing the confidence of the two schemes through a weighting parameter $0\leq\tau\leq1$. 
When $\tau$ is neither $0$ nor $1$, the proposed algorithm becomes a semi-supervised transfer learning algorithm. 
So we define the following empirical risk:
\[
\begin{aligned}
\hat{R}_{\epsilon,\tau}^{\vect{\alpha}}(u,v) &\eqdef \tau(1-\epsilon)\hat{R}_{\Tcal}(u,v) + \tau\epsilon\hat{R}_{\Scal^{\vect{\alpha}}}(u,v)\\
& + \tau\epsilon\hat{\mathbf{W}}_1(\tilde{\Tcal}_u, \tilde{\Scal}^{\vect{\alpha}}_u) + (1-\tau)\hat{\mathbf{W}}_1(\tilde{\Tcal}_{u,v}, \tilde{\Scal}_u^{\vect{\alpha}})
\end{aligned}
\]
\paragraph{Choice of $\tilde{\Fcal}$}
Since $\ell(h(v,\cdot), \cdot), \forall v \in \Vcal$ satisfies the Lipschitz condition with $|\ell(h(v,\tilde{x}), y) - \ell(h(v, \tilde{x}^\prime), y^\prime)| \leq \rho_{\tilde{z}}(\tilde{z}, \tilde{z}^\prime)$ for the $\rho_{\tilde{z}}$ used in the two Wasserstein distances, we can choose $\tilde{\Fcal} = \{\tilde{f} = \ell(h(v^\prime,\cdot), \cdot): v^\prime\in \Vcal^\prime=\Vcal\}$. As discussed in Section 4.2, We assume that the supremum over the space of all the 1-Lipschitz functions $\Fcal$ is attained with some $\tilde{f}\in\tilde{\Fcal}$.
Thus, the model contains two predictors $h(v, .)$ and $h(v^\prime, .)$ using the same network structure ($\Vcal^\prime=\Vcal$) that make predictions from the feature representation $g(u,x)$. The former minimize the empirical risk $\tau(1-\epsilon)\hat{R}_{T}(u,v) + \tau\epsilon\hat{R}_{S^{\vect{\alpha}}}(u,v)$ w.r.t $v$. The other auxiliary predictor maximizes the empirical estimation of the two Wasserstein distances w.r.t $v^\prime$. Then we minimize the sum of these empirical estimates w.r.t $u,v$. The above approach is different from the marginal representation alignment method, where the $\mathbf{W}_1$ distance is defined on $\tilde{\Xcal}$ and the corresponding $\tilde{f}$ is the domain discriminator adopted by DANN-like methods \citep{wen2020domain, shui2021aggregating}.


\paragraph{Optimize the Wasserstein Distances}\label{choice_f} With the choice of $\tilde{f}$ described above, the empirical risk of the pseudo target distribution is:
$$\hat{R}_{\Tcal_{u,v}}(u,v, v^\prime) = \frac{1}{m_t^\prime} \sum_{j=1}^{m_t^\prime} \ell(h(v^\prime, g(u,X_j^t)), h(v, g(u, X_j^t)))$$
Therefore, we have the following approximation of the two Wasserstein distances:

\hspace{0.5cm}$\hat{\mathbf{W}}_1(\tilde{\Tcal}_u, \tilde{\Scal}^{\vect{\alpha}}_u) = \max_{v^\prime}[ \hat{R}_{\Tcal}(u, v^\prime) - \hat{R}_{\Scal^{\vect{\alpha}}}(u,v^\prime)]$

\hspace{0.5cm}$\hat{\mathbf{W}}_1(\tilde{\Tcal}_{u,v}, \tilde{\Scal}^{\vect{\alpha}}_u) = \max_{v^\prime}[ \hat{R}_{\Tcal_{u,v}}(u,v, v^\prime) - \hat{R}_{\Scal^{\vect{\alpha}}}(u,v^\prime)]$

Then we get the two mini-max optimization objectives $\min_u\hat{\mathbf{W}}_1(\tilde{\Tcal}_u, \tilde{\Scal}^{\vect{\alpha}}_u)$ and $\min_{u,v}\hat{\mathbf{W}}_1(\tilde{\Tcal}_{u,v}, \tilde{\Scal}^{\vect{\alpha}}_u)$, which can be implemented with a gradient reversal layer \citep{ganin2016domain}.

\paragraph{Joint Optimization with Stochastic Gradient Algorithm}
\cite{pensia2018generalization} provides methods bounding the mutual information with gradient updates for noisy iterative algorithms, which is especially suitable for deep learning, where the Stochastic Gradient Descent (SGD) and its variants are often applied. We apply this technique to analyze the aforementioned multi-source transfer scheme using the Stochastic Gradient Langevin Dynamics (SGLD) algorithm \citep{welling2011bayesian} --- a stochastic variant of SGD with independent noise injection at each step. 

At each iteration $k$, a batch of source data $S_{B_{1:N}}^k$ is sampled from the source datasets $S_{1:N}$. Simultaneously, a batch of labeled target data $T_{B}^k$ and a batch of unlabeled target data $T_{X_{B}}^k$ are sampled from the corresponding target datasets. Subsequently, we update the representation parameter $U$ and predictor parameter $V$ with the following updates:
$$U_k = U_{k-1} - \eta_u^k G_u^k + \xi_u^k;\quad V_{k} = V_{k-1} - \eta_v^k G_v^k + \xi_v^k\, ,$$
where $\xi_k^u \sim N(0, \sigma_k^2 I_{d_u})$ and $\xi_k^v \sim N(0, \sigma_k^2 I_{d_v})$ are injected isotropic Gaussian noise of variance $\sigma_k^2$. Parameters $\eta_u^k$ and $\eta_v^k$ are, respectively, the learning rate for $U$ and $V$ at each step. $G_u^k$ and $G_v^k$ are, respectively, the gradient estimation of $\nabla_U \hat{R}_{\epsilon,\tau}^{\vect{\alpha}}(U,V)$ and $\nabla_V \hat{R}_{\epsilon,\tau}^{\vect{\alpha}}(U,V)$ using the batch datasets sampled at each iteration. Note that the sampling strategy needs to be agnostic to the previous iterates of the parameters.

\subsection{Theoretical Guarantee via Gradient Norm Bound}

We now present the gradient norm bound for the proposed IMDA algorithm. 
\begin{theorem}\label{thm:gen_gradient_norm}
Following Theorem~\ref{thm:sup_wasserstein},~\ref{thm:pseudo_label} ~\ref{gen_supervised},~\ref{gen_un_pseudo}, define a weight parameter $\tau$ that balances the unsupervised and supervised MDA. Adopt the aforementioned choice of $\tilde{\Fcal}$ with $\sigma^\prime=\sigma$ and the SGLD updates described above. Then we can obtain the gradient norm bound for the expected target risk, $\forall \vect{\alpha}\in\Delta_N, \tau,\epsilon \in [0,1]$, we have:
\[
\begin{aligned}
\EE&_{U,V,S^{\vect{\alpha}}, T, T_X^\prime} R_{\Tcal}(U, V) \leq \EE_{U,V, S^{\vect{\alpha}}, T, T_X^\prime} \hat{R}_{\epsilon,\tau}^{\vect{\alpha}}(U,V)\\
&+ \tau \sigma\sqrt{2(\frac{(1-\epsilon)^2}{m_t} + \epsilon^2\sum_{i=1}^N\frac{\alpha_i^2}{m_i})(\delta_u + \delta_v)}\\
&+ \tau \epsilon \sigma \sqrt{2(\sum_{i=1}^N\frac{\alpha_i^2}{m_i} + \frac{1}{m_t}) \delta_u} + (1-\tau)\EE_{U,V} R_{rep}^*(U,V)\\
&+ (1-\tau)(\sigma\sqrt{2(\sum_{i=1}^N\frac{\alpha_i^2}{m_i} + \frac{1}{m_t^\prime})(\delta_u + \delta_v)} + R^*)\,,
\end{aligned}
\]
$\delta_{u}\eqdef\sum_{k=1}^K \frac{(\eta_u^k)^2\EE\|G_u^k\|_2^2 }{2\sigma_k^2}$ and $\delta_v\eqdef\sum_{k=1}^K \frac{(\eta_v^k)^2\EE\|G_v^k\|_2^2}{2\sigma_k^2}$ are the accumulated gradient  norm for $U$ and $V$, respectively.
\end{theorem}
Theorem~\ref{thm:gen_gradient_norm} is quite general. When $\tau=1$, it covers the supervised MDA. When $\tau=0$, it covers the unsupervised setting. In other cases, it can be applied for semi-supervised transfer learning. We can see that for semi-supervised transfer, the assumption for a small $R^* + \EE_{U,V}R^*_{rep}(U,V)$ is weakened. When we have many labeled target data ($m_t$ large), we can choose a large $\tau$. Moreover, if we have only one source, $\vect{\alpha}$ becomes a scalar and equals $1$. Thus the bound can also be applied to single-source domain adaptation. See the proof in Appendix~\ref{proof:gradient_bound}.

\subsection{Implementation and Discussion}

\begin{table*}[hbt!]
\caption{Accuracy(\%) of Unsupervised MDA on Target Shift Data (Drop Rate $50\%$): Amazon Review(Left), Digits(Right)}\label{tab:umda}
\begin{minipage}{\columnwidth}
    \centering
\begin{adjustbox}{width=\textwidth}
    \begin{tabular}{c|cccc|c}
&   \multicolumn{4}{c|}{Target Domain} &  \\
Method & Books & DVD &  Electronics & Kitchen  & Average  \\ 
\hline
Source &   68.15$_{\pm 1.37}$   &  69.51$_{\pm 0.74}$ &      82.09$_{\pm 0.88}$  &  75.30$_{\pm 1.29}$ &  73.81 \\ 
DANN   &  65.59$_{\pm 1.35}$    & 67.23$_{\pm 0.71}$ &  80.49$_{\pm 1.11}$     & 74.71$_{\pm 1.53}$  &   72.00 \\ 
MDAN   &    68.77$_{\pm 2.31}$  & 67.81$_{\pm 2.46}$ &  80.96$_{\pm 0.77}$   & 75.67$_{\pm 1.96}$  & 73.30 \\ 
MDMN   & 70.56$_{\pm 1.05}$  &69.64$_{\pm 0.73}$     &  82.71$_{\pm 0.71}$    & 77.05$_{\pm 0.78}$  &  74.99\\ 
M$^3$SDA  & 69.09$_{\pm 1.26}$ & 68.67$_{\pm 1.37}$ & 81.34$_{\pm 0.66}$ & 76.10$_{\pm 1.47}$  & 73.79 \\ 
\hline
DARN   &69.88$_{\pm 1.91}$  & 69.63$_{\pm 1.09}$ &80.83$_{\pm 1.13}$      & 77.47$_{\pm 1.05}$   & 74.45  \\ 
WADN   &    \textbf{75.74}$_{\pm 1.60}$  & \textbf{78.46}$_{\pm 1.72}$ & 82.10$_{\pm 2.09}$    & 81.05$_{\pm 2.25}$   & 79.34 \\ 
IMDA(Ours) & 75.21$_{\pm2.3}$ & 78.15$_{\pm1.79}$ &\textbf{83.93}$_{\pm0.24}$ & \textbf{81.44}$_{\pm 1.44}$ & \textbf{79.68}
\end{tabular}
\end{adjustbox}
\end{minipage}\hfill
\begin{minipage}{\columnwidth}
    \centering
\begin{adjustbox}{width=\textwidth}
\begin{tabular}{c|cccc|c}
&   \multicolumn{4}{c|}{Target Domain} &  \\
Method & MNIST & SVHN  &  SYNTH  & USPS  & Average  \\
\hline
Source &  84.93$_{\pm 1.50}$   &  67.14$_{\pm 1.40}$   &   78.11$_{\pm 1.31}$   & 86.02$_{\pm 1.12}$ & 79.05 \\ 
DANN   &  86.99$_{\pm 1.53}$   &  69.56$_{\pm 2.26}$   &   78.73$_{\pm 1.30}$   & 86.81$_{\pm 1.74}$ & 80.52 \\ 
MDAN   &  87.86$_{\pm 2.24}$   & 69.13$_{\pm 1.56}$     &    79.77$_{\pm 1.69}$  & 86.50$_{\pm 1.59}$ & 80.81 \\ 
MDMN   &  87.31$_{\pm 1.88}$     &  69.84$_{\pm 1.59}$    & 80.27$_{\pm 0.88}$     & 86.61$_{\pm 1.41}$  & 81.00 \\ 
M$^3$SDA  &  87.22$_{\pm 1.70}$   &  68.89$_{\pm 1.93}$   &   80.01$_{\pm 1.77}$   & 86.39$_{\pm 1.68}$ & 80.87 \\ \hline
DARN   &   86.96$_{\pm 1.27}$       &   68.99$_{\pm 1.76}$   &  80.65$_{\pm 1.12}$  &86.89$_{\pm 1.64}$  &  80.87\\ 
WADN   &   89.04$_{\pm 0.86}$      &  71.68$_{\pm 1.27}$     &  \textbf{82.06}$_{\pm 0.97}$   & \textbf{90.06}$_{\pm 1.12}$  & \textbf{83.21} \\ 
IMDA(Ours) &  \textbf{89.26}$_{\pm 0.89}$      &  \textbf{76.6}$_{\pm 0.86}$     &  81.76$_{\pm 0.69}$   & 84.45$_{\pm 1.2}$  & 83.02 
\end{tabular}
\end{adjustbox}
\end{minipage}\hfill 
\end{table*}

\paragraph{Optimize the Domain Weights $\vect{\alpha}$}
From Theorem~\ref{thm:gen_gradient_norm}, for given $u,v,v^\prime$ and $\epsilon,\tau$, we can optimize $\vect{\alpha}$ w.r.t the related term in the upper bound. So we get the following optimization problem with constants $C_0, C_1>0$:
\[
\begin{aligned}
\min_{\vect{\alpha}} \space &\left( (\epsilon\tau + C_0(1-\tau)) \hat{R}_{\Scal^{\vect{\alpha}}}(u,v) \right. \\
&\quad\quad - (\epsilon\tau + 1- \tau) \hat{R}_{\Scal^{\vect{\alpha}}}(u,v^\prime)  \\
&\left.\quad\quad + C_1((1-\tau + \tau\epsilon)\sqrt{\delta_u + \delta_v} + \tau\epsilon \sqrt{\delta_u}) R(\vect{\alpha}) \right)\, ,\\
&R(\vect{\alpha}) = \sqrt{\sum_{i=1}^N \frac{\alpha_i^2}{m_i}},\ \text{s.t.} \forall i\in[N], \alpha_i\geq 0, \sum_{i=1}^N \alpha_i = 1\,,
\end{aligned}
\]
which can be solved with a standard convex optimization toolbox. We fix $\vect{\alpha}$ at each training epoch to optimize the network parameters $u,v, v^\prime$ by minimizing the empirical risk $\hat{R}_{\epsilon,\tau}^{\vect{\alpha}}(u,v)$. Then, at the end of each epoch, we optimize $\vect{\alpha}$ given the updated $u,v, v^\prime$ with the above convex optimization objective. We add $C_0\hat{R}_{\Scal^{\vect{\alpha}}}(u,v)$ for unsupervised MDA ($\tau=0$) as approximation of $R_{\Scal^{\vect{\alpha}}}(u^*,v^*)$ in $R^*$ when optimizing $\vect{\alpha}$, which does not exist in the optimization objective for $u,v$. Other higher order terms of $\vect{\alpha}$ in $\delta_u, \delta_v$ are ignored. 

We emphasize that the above objective is different from \citet{shui2021aggregating}, which also optimizes the domain weights $\vect{\alpha}$. The coefficient before the regularization term $R(\vect{\alpha})$ adapts with the gradient steps, which has an interesting interpretation. When the gradient step grows, both the model parameters and the domain weights $\vect{\alpha}$ may over-fit to the existing datasets. Thus, the adaptive regularization coefficient of $R(\vect{\alpha})$ can gradually control the possibility of over-fitting to specific source distributions.   

\paragraph{Complexity Comparison} For memory complexity, the proposed method only requires \textbf{one} duplicate predictor. In contrast, WADN \citep{shui2021aggregating} requires $N$ domain discriminator and $N|\Ycal|$ class feature centroids. MDAN and DARN \citep{wen2020domain} need $N$ domain discriminators. M$^3$SDA \citep{peng2019moment} and MDMN \citep{li2018extracting} need $N^2$ domain discriminators. For time complexity, the proposed method needs to compute $\Ocal(N)$ mini-batch risks and the corresponding backward gradients at each batch for optimizing the model parameters, which is the same as the other methods. Moreover, after each epoch, we need to optimize the domain weights with $\Ocal(N)$ time complexity, which is an improvement compared to the $\Ocal(N|\Ycal|)$ time complexity of WADN, where the $|\Ycal|$ comes from estimating the label ratio. 

\begin{table*}[hbt!]
\caption{Accuracy(\%) of Supervised MDA on Target Shift Data (Drop Rate $50\%$): Amazon Review(Left), Digits(Right)}\label{tab:mda}
\begin{minipage}{\columnwidth}
    \centering
\begin{adjustbox}{width=\textwidth}
\begin{tabular}{c|cccc|c}
&   \multicolumn{4}{c|}{Target Domain} &  \\
Method & Books & DVD &  Electronics & Kitchen  & Average  \\ 
\hline
Source + Tar &    72.59$_{\pm 1.89}$ &  73.02$_{\pm 1.84}$    & 81.59$_{\pm 1.58}$    & 77.03$_{\pm 1.73}$  & 76.06  \\ 
DANN   &    67.35$_{\pm 2.28}$       &   66.33$_{\pm 2.42}$      & 78.03$_{\pm 1.72}$     & 74.31$_{\pm 1.71}$  & 71.50  \\ 
MDAN   &    68.70$_{\pm 2.99}$       &    69.30$_{\pm 2.21}$       & 78.78$_{\pm 2.21}$ & 74.07$_{\pm 1.89}$  & 72.71  \\ 
MDMN   &    69.19$_{\pm 2.09}$       &     68.71$_{\pm 2.39}$        & 81.88$_{\pm 1.46}$  & 78.51$_{\pm 1.91}$  & 74.57  \\ 
M$^3$SDA  &  69.28$_{\pm 1.78}$ &    67.40$_{\pm 0.46}$     & 76.28$_{\pm 0.81}$    & 76.50$_{\pm 1.19}$  & 72.36  \\ 
RLUS   &    71.83$_{\pm 1.71}$      &    69.64$_{\pm 2.39}$     & 81.98$_{\pm 1.04}$    & 78.69$_{\pm 1.15}$  & 75.54 \\ 
MME   &    69.66$_{\pm 0.58}$  &   71.36$_{\pm 0.96}$  &   78.88$_{\pm 1.51}$   &  76.64$_{\pm 1.73}$  &  74.14 \\
\hline
DARN   & 69.48$_{\pm 2.28}$     &    69.59$_{\pm 2.90}$     & 80.66$_{\pm 1.38}$    & 77.25$_{\pm 0.94}$  & 74.25 \\
WADN   & 74.52$_{\pm 1.45}$     &  \textbf{77.39}$_{\pm 1.04}$    & 81.40$_{\pm 1.61}$    & \textbf{81.52}$_{\pm 1.64}$  & 78.71 \\ 

IMDA(Ours) & \textbf{75.17}$_{\pm 0.94}$ & 77.92$_{\pm 1.73}$ & \textbf{83.04}$_{\pm 1.34}$ & 80.11$_{\pm 1.17}$ & \textbf{79.06}
\end{tabular}
\end{adjustbox}
\end{minipage}\hfill 
\begin{minipage}{\columnwidth}
    \centering
\begin{adjustbox}{width=\textwidth}
\begin{tabular}{c|cccc|c}
&   \multicolumn{4}{c|}{Target Domain} &  \\
Method & MNIST & SVHN  &  SYNTH  & USPS  & Average  \\
\hline
Source + Tar &   79.63$_{\pm 1.74}$      & 56.48$_{\pm 1.90}$       & 69.64$_{\pm 1.38}$     & 86.29$_{\pm 1.56}$  &  73.01 \\ 
DANN   & 86.77$_{\pm 1.30}$              & 69.13$_{\pm 1.09}$       & 78.82$_{\pm 1.35}$     & 86.54$_{\pm 1.03}$  &  80.32 \\
MDAN   & 86.93$_{\pm 1.05}$              & 68.25$_{\pm 1.53}$       & 79.80$_{\pm 1.17}$     & 86.23$_{\pm 1.41}$  &  80.30 \\
MDMN   & 77.59$_{\pm 1.36}$              & 69.62$_{\pm 1.26}$       & 78.93$_{\pm 1.64}$     & 87.26$_{\pm 1.13}$  &  78.35 \\
M$^3$SDA  & 85.88$_{\pm 2.06}$           & 68.84$_{\pm 1.05}$       & 76.29$_{\pm 0.95}$     & 87.15$_{\pm 1.10}$  &  79.54 \\
RLUS   &  87.61$_{\pm 1.08}$             & 70.50$_{\pm 0.94}$       & 79.52$_{\pm 1.30}$     & 86.70$_{\pm 1.13}$  &  81.08  \\ 
MME   & 87.24$_{\pm 0.95}$  &  65.20$_{\pm1.35}$  &  80.31$_{\pm 0.60}$  &  87.88$_{\pm 0.76}$  &  80.16 
 \\
 \hline
 DARN   & 86.56$_{\pm 1.48}$              & 68.76$_{\pm 1.36}$       & 80.4$_{\pm 1.27}$     & 86.82$_{\pm 1.09}$  &  80.64 \\
WADN   & \textbf{88.02}$_{\pm 1.27}$     & 70.54$_{\pm 1.06}$       & \textbf{81.58}$_{\pm 0.96}$     & 90.48$_{\pm 1.11}$  &  82.66  \\
IMDA(Ours) & 84.82$_{\pm 1.15}$ & \textbf{75.46}$_{\pm 1.2}$ & 79.87$_{\pm 1.04}$ & \textbf{91.6}$_{\pm 0.96}$ & \textbf{82.94}
\end{tabular}
\end{adjustbox}
\end{minipage}
\end{table*}

\begin{figure*}[hbt!]
    \centering
    \subfloat[Supervised MDA]{
     \includegraphics[width=0.29\linewidth]{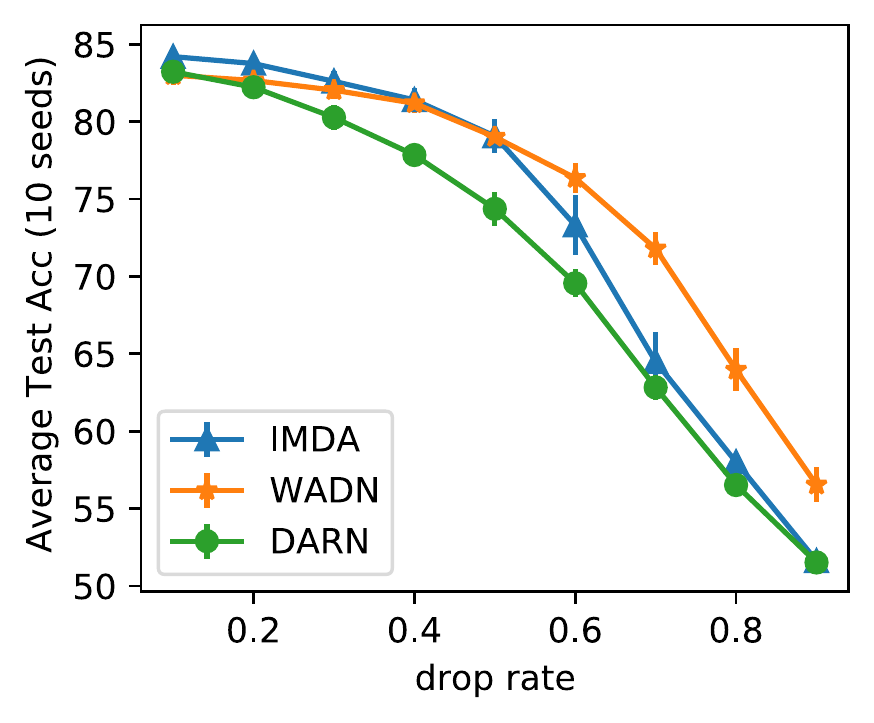}}
    \subfloat[Unsupervised MDA]{
     \includegraphics[width=0.29\linewidth]{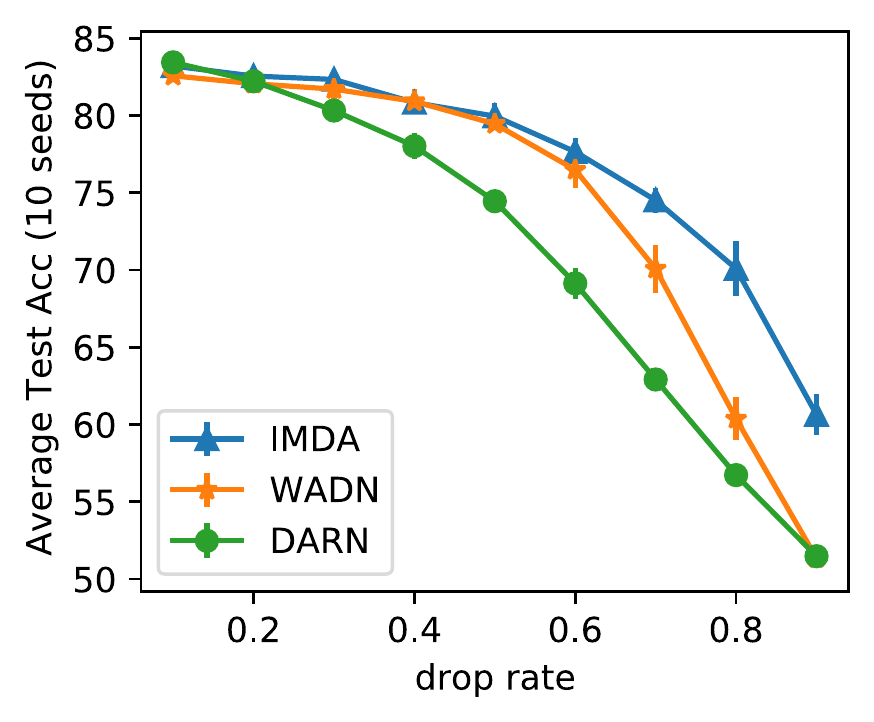}}
 \subfloat[Domain Weights and Regularization]{
   \includegraphics[width=0.38\linewidth]{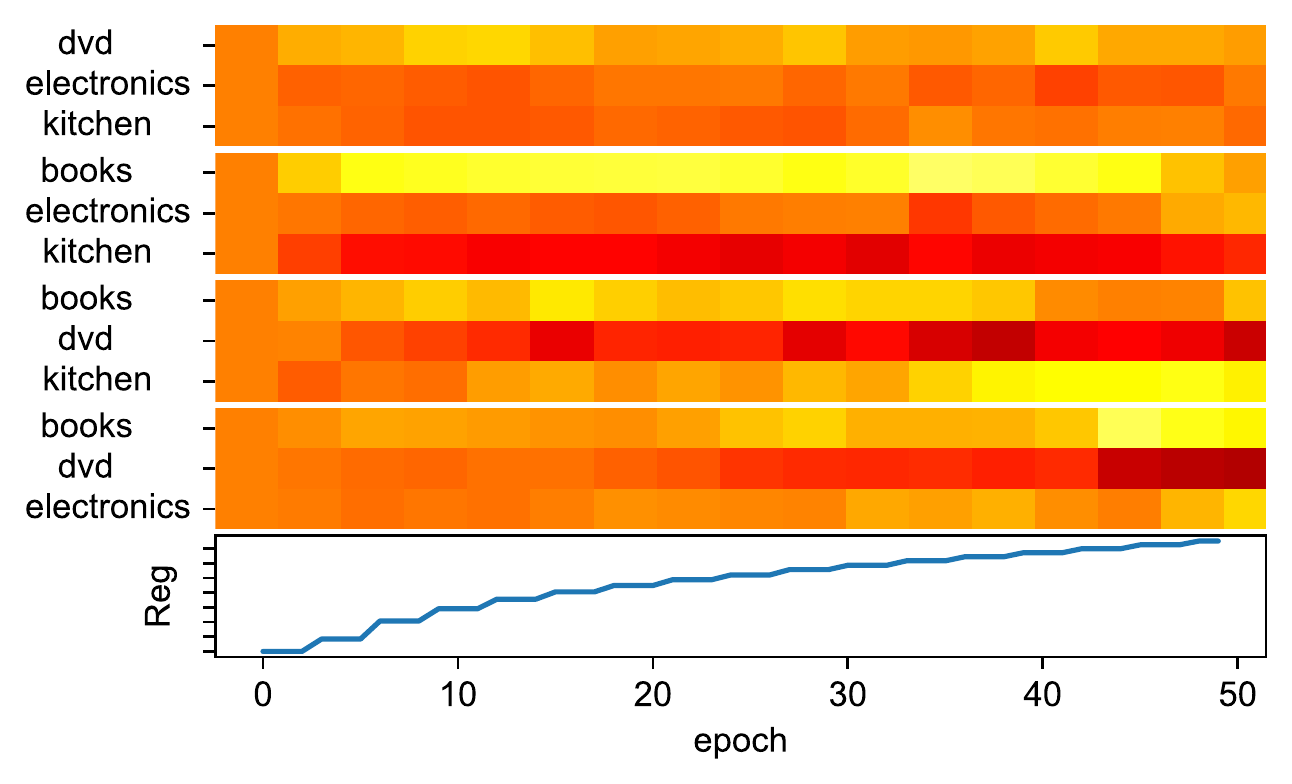}
  }
\caption{Average Test Accuracy over All Domains w.r.t Different Levels of Target Shift on Amazon Dataset for (a) Supervised MDA with Few Target Labels and (b) Unsupervised MDA with Pseudo Labels. (c) Visualization of Domain Weights and Regularization Coefficients in Different Training Epochs on the Amazon Dataset for Supervised MDA with a Drop Rate of $0.4$.}
    \label{fig:amazon}
\end{figure*}

\section{EXPERIMENTS}

In this section, we test IMDA (Sec.~\ref{algo}) on the target shift MDA benchmark proposed by \citet{shui2021aggregating}. Experimental details and additional results can be found in Appendix \ref{sec:exp_details} and \ref{sec:add_exp}.

\subsection{Datasets and Baselines}

\paragraph{Amazon Review Dataset} \citet{blitzer2007biographies} consists of positive and negative reviews from four domains: "Books", "DVD", "Electronics" and "Kitchen". We use each domain in turns as the target domain and the other three as the source domains. The data are pre-processed with the same approach as \citet{chen2012marginalized}, where the top-5000 frequent unigrams/bigrams are used as the bag-of-words features. The original Amazon Review data is label balanced. To create the \textit{target shift} data, we follow the strategy of \citet{shui2021aggregating} by randomly dropping $50\%$ negative reviews in the source domains while keeping the target domain unchanged.

\paragraph{Digits Dataset} contains four different tasks for digit recognition. MNIST \citep{lecun1998mnist} and  USPS \citep{hull1994database} are standard handwriting digits datasets. Street View House Number (SVHN) \citep{netzer2011reading} is a digits dataset consisting of Google Street View images of house numbers. SYNTH \citep{ganin2015unsupervised} is a synthetic dataset applying various transformations on SVHN. As \citet{shui2021aggregating}, we also drop $50\%$ of the data on digits 5-9 of all the sources with the target label distribution unchanged to create the \textit{target shift} data.

\paragraph{Baselines} We first introduce the baselines for unsupervised MDA. \textbf{Source} use all the labeled source datasets to train the predictor without adaptation. \textbf{DANN} \citep{ganin2016domain} is a single-source DA method, thus conducting marginal distribution matching between the merged source data and the target data. \textbf{MDAN} \citep{zhao2018adversarial}, \textbf{MDMN} \citep{li2018extracting}, \textbf{M$^3$SDA} \citep{peng2019moment} and \textbf{DARN} \citep{wen2020domain} are marginal distribution matching methods. \textbf{WADN} \citep{shui2021aggregating} first estimates the label ratio for correcting the target shift, then align the conditional distributions. Thus it also performs a (distinct) joint distribution matching method. For supervised MDA, \textbf{Source + Tar} is trained on merging the source data and the few labeled target data. Two more baselines, \textbf{RLUS} \citep{konstantinov2019robust}
and \textbf{MME} \citep{saito2019semi}, are also considered.

\subsection{Results and Analysis}
\paragraph{Results on Benchmark}
We compare our empirical results on the datasets mentioned above with the baselines provided by \citet{shui2021aggregating}. For supervised MDA, we consider the case where only a few labeled target samples are accessible. So we randomly sample $10\%$ labeled target data as the training set, and the rest $90\%$ are used as the test set. For unsupervised MDA, the above target data without labels are provided for training, and the results are tested on another unseen target dataset. We have \textit{reproduced} the results of WADN and DARN using the code provided in these two papers and directly used the results presented in \cite{shui2021aggregating} for the other baselines. The comparison is shown in Tab.~\ref{tab:umda} for unsupervised MDA and Tab.~\ref{tab:mda} for supervised MDA. 
The reproduced results and those obtained by the proposed approach (IMDA) are averaged over ten runs. Each approach's means and standard deviations are reported with the best value shown in bold. 

From the tables, we can see that IMDA has improved on "Electronics" and "Kitchen" compared to the state-of-the-art (WADN) while having a slight drop on "Books" and "DVD" for unsupervised MDA. For supervised MDA, it improves on "Electronics" and "Books." The average Accuracy has a $0.3-0.4\%$ \textit{improvement} w.r.t the WADN. For the Digits dataset, the proposed approach has significant improvement on "SVHN" while a drop is shown on either "MNIST" or "USPS" for supervised and unsupervised MDA. W.r.t the average accuracy, we have \textit{comparable results} to WADN with a slight improvement on supervised MDA and a slight decrease on unsupervised MDA. In addition, the worse performance of \textbf{Source + Tar} compared to all the divergence-based methods validates our statement for exploiting the target data in Theorem~\ref{gen_supervised}.

\paragraph{Ablation Study for Target Shift} We conduct the ablation study on the Amazon dataset to show the effectiveness of IMDA for mitigating the target shift. We adjust the drop rate for negative reviews from $10\%-90\%$. The results are shown in Fig.~\ref{fig:amazon} (a) and (b) for supervised MDA with few target labels and unsupervised MDA with pseudo labels, respectively. 

Fig~\ref{fig:amazon} (b) illustrates that IMDA can tackle significant target shift for unsupervised MDA and outperforms WADN, which needs, additionally, an estimate of the label ratio with the pseudo labels. The performance of WADN drops because the predictor is often inaccurate under a significant target shift due to an incorrect label ratio estimate, which in turn degrades the predictor's performance. 

For the supervised setting, only a \textit{few labeled target samples} are provided. IMDA performs as well as WADN with a drop rate smaller than $50\%$, while its performance degrades when the drop rate goes up. The result is consistent with the proposed theory. When the source and target joint distributions differ significantly w.r.t the label distribution, correcting the target shift with simple statistics for the label ratio given the true target labels is more efficient than minimizing the Wasserstein distance, which requires adequate data. 
\paragraph{Visualization of Domain Relations} We show, in the heat-map of Fig.~\ref{fig:amazon}(c), the evolution of domain weights and the regularization coefficient w.r.t the training epoch, where a darker color means a larger weight. Moreover, we can see an adaptive effect of avoiding over-fitting to specific domains, \eg, the weight of "DVD" was tuned down in some epochs when estimating "Electronics." 

\section{DISCUSSION AND CONCLUSION} 

\paragraph{Limitation} \citet{kumar2020understanding} and \citet{wang2022continual} conduct unsupervised DA with self-training when the target domain continuously shifts (gradual DA /continuous test-time DA). In this scenario, there often exists the phenomenon of error accumulation. IMDA also suffers from this issue due to the utilization of pseudo labels. \citet{wang2022continual} proposes two methods to mitigate the error accumulation -- the weight-averaged and augmentation-average pseudo labels. A similar weight-averaged approach can be used in IMDA to improve the algorithm's stability. 

\paragraph{Future Application in High-level Computer Vision Tasks} Complex computer vision tasks like semantic segmentation often require structural adaptation. So extending the traditional feature alignment methods needs a specific feature map design. E.g., ADVENT \citep{vu2019advent} uses the weighted self-information space. However, IMDA can be directly extended to the segmentation task using the mini-max objective w.r.t the empirical risks, where the structural information is contained in the definition of $\Ycal$. 

\paragraph{Future Application in Fairness} Recently, the ethical problems related to machine learning algorithms have gained growing concern in the community, especially regarding how to avoid implicit discrimination. Since unfair predictions are often associated with distribution shifts, domain adaptation methods are considered for improving fairness in \cite{zhao2019inherent, schumann2019transfer}. The proposed approach can also be applied or incorporated with other fair learning algorithms \citep{hardt2016equality, barocas2017fairness, liu2019implicit, shui2022fair, shui2022learning}.

 To conclude, this paper conducts an information-theoretic analysis of representation learning for both supervised and unsupervised MDA. We first provide fully algorithm-dependent bounds for MDA that apply to deep learning-based algorithms. Then, we propose a novel algorithm (IMDA) that mitigates the target shift and empirically achieves comparable performance to the previous state-of-the-art with improved memory efficiency. Notably, the proposed algorithm outperforms the previous works w.r.t a significant target shift in the more practically realistic unsupervised learning scenario. Finally, the proposed approach is fundamental and can be easily extended to other applications, \eg,  complex cv tasks and fairness. 

\subsubsection*{Acknowledgements}
We appreciate constructive feedback from anonymous reviewers and meta-reviewers. This work is supported by the Natural Sciences and Engineering Research Council of Canada (NSERC) Discovery Grant, the Collaborative Research and Development Grant from SSQ Assurances and NSERC, and the China Scholarship Council.






\bibliography{aistats}
\bibliographystyle{plainnat}

\appendix
\onecolumn

\section{DEFINITIONS}
\begin{definition}\label{def:coupling}(Coupling).
Let $(\Xcal,\mu)$ and $(\Ycal, \nu)$ be two probability spaces. Coupling $\mu, \nu$ means constructing two random variables $X$ and $Y$ on some probability space $(\Zcal, \pi)$, such that $\Zcal=\Xcal\times\Ycal$, $(\text{proj}_{\Xcal})_{\#}\pi = \mu$ and $(\text{proj}_{\Ycal})_{\#}\pi = \nu$, which means that $\pi$ is the joint measure on $\Xcal\times\Ycal$ with marginals $\mu, \nu$ on $\Xcal$ and $\Ycal$ respectively. The couple $(X,Y)$ is called a coupling of $(\mu,\nu)$.  
\end{definition}

\begin{definition}(Lipschitzness).\label{def:lip}
Given two metric spaces $(\Xcal, \rho_x)$, $(\Ycal, \rho_y)$, a function $f:\Xcal\rightarrow \Ycal$ is $L$-Lipschitz w.r.t $\Xcal$ if $\forall x, x^\prime \in \Xcal$, $\rho_y(f(x), f(x^\prime)) \leq L \rho_{x}(x, x^\prime)$. Specifically, if $\Xcal \subset \Reals^{d_1}, \Ycal \subset \Reals^{d_2}$, we often choose $\rho_x(x,x^\prime) = \|x - x^\prime\|$ and $\rho_y(\cdot,\cdot) = \|y - y^\prime\|$ w.r.t norm $\|\cdot\|$.
\end{definition}

\begin{definition}(Transportation Cost Inequality).
A probability measure $\mu$ on $(\Xcal, \rho)$ satisfies an $L^p$ transportation inequality with constant $c>0$, if for every probability $\nu \ll \mu$, we have:
\[
\mathbf{W}_p(\mu, \nu)\leq \sqrt{2c\KL(\nu\|\mu)}
\]
\end{definition}

\begin{definition}(Sub-Gaussian).\label{def:subgaussian}
 Define the cumulant generating function(CGF) of random variable $X$ as $\psi_X(\lambda) \eqdef \log\mathbb{E}[e^{\lambda(X - \mathbb{E}[X])}]$. $X$ is said to be $\sigma$-sub-Gaussian if
\begin{equation*}
    \psi_X(\lambda) \leq \frac{\lambda^2\sigma^2}{2}, \forall\lambda \in \mathbb{R}\, .
\end{equation*}
\end{definition}

\begin{definition}(Mutual Information).\label{def:mi}
Let $X$ and $Y$ be arbitrary random variables and $\KL$ denote the KL divergence. The mutual information between $X$ and $Y$ is defined as:
\[I(X;Y)\eqdef\KL(P(X,Y)\|P(X)P(Y))\]
\end{definition}

\begin{definition}(Wasserstein-$p$ Distance). \label{def:def_wasserstein}
Let the two distributions defined on the same Polish metric space $(\Xcal,\rho)$, where $\rho(.,.)$ is a metric and $p\in [1, +\infty)$, $\Pi(\mu,\nu)$ is the set of all the couplings (see Definition \ref{def:coupling}) of $\mu,\nu$. The Wasserstein distance with order $p$ between $\mu$ and $\nu$ is defined as:
\[
\begin{aligned}
\mathbf{W}_p(\mu,\nu) &\eqdef \inf_{\pi \in \Pi(\mu, \nu)} \left[ \int_{\Xcal\times \Xcal} \rho(x, x^\prime)^p d\pi(x,x^\prime) \right]^{1/p}\,.
\end{aligned}
\]
\end{definition}
If $\rho(x, x^\prime) = \|x - x^\prime\|$ and $p=1$, the above definition is the Earth-Mover(EM) or Wasserstein-$1$ distance.

\vfill

\section{TECHNICAL LEMMAS}\label{Donsker}
\begin{lemma}\textbf{(Donsker-Varadhan representation)} Let $P$ and $Q$ be two probability measures defined on a set $\mathcal{X}$. Let $\mathfrak{g} : \mathcal{X} \rightarrow \Reals$ be a measurable function, and let $\mathbb{E}_{x \sim Q}[\exp{\mathfrak{g}(x)}] \leq \infty$. Then we have \[D_{\text{KL}}(P||Q) = \sup\limits_{\mathfrak{g}}\{\mathbb{E}_{x\sim P}[\mathfrak{g}(x)] - \log\mathbb{E}_{x \sim Q}[\exp{\mathfrak{g}(x)}]\}.\]
\end{lemma}

The proof of the above Lemma can be found in Corollary 4.15 \citep{boucheron2013concentration} or Theorem 4.1 \citep{duchi2016lecture}.


\begin{lemma}\label{lemma_sub_gaussian}
Let $X_1,X_2,...X_n$ be independent $\sigma_i$-sub-Gaussian random variables. Then
\[
\EE \left[\exp\left(\lambda\sum_{i=1}^n (X_i -\EE[X_i])\right)\right] \leq \exp\left(\frac{\lambda^2\sum_{i=1}^n \sigma_i^2}{2}\right), \forall \lambda \in \Reals
\]
that is, $\sum_{i=1}^n X_i$ is $\sqrt{\sum_{i=1}^n \sigma_i^2}$-sub-Gaussian. 
\end{lemma}

\begin{proof}
$X_i, \forall i=\{1,2,...,n\}$ is $\sigma_i$-sub-Gaussian, from Definition~\ref{def:subgaussian}, we have $\EE[\exp{(\lambda(X_i-\EE X_i))}] \leq \exp{(\frac{\lambda^2\sigma_i^2}{2})}$.

Since the variables are independent, we have
\[
\EE \left[\exp\left(\lambda\sum_{i=1}^n (X_i -\EE[X_i])\right)\right] = \prod_{i=1}^n \EE[\exp{(\lambda(X_i-\EE X_i))}] \leq \prod_{i=1}^n \exp{(\frac{\lambda^2\sigma_i^2}{2})} = \exp\left({\frac{\lambda^2\sum_{i=1}^n\sigma_i^2}{2}}\right)
\]
\end{proof}

\begin{lemma}\label{lemma_entropy}
For independent variables $X$ and $Y$, we have the following entropy inequality:
\[\max\{H(X),H(Y)\} \leq H(X+Y) \leq H(X,Y) = H(X) + H(Y)\]
\end{lemma}

\begin{proof}
Let $Z=X+Y$, from the definition of mutual information, we can obtain the well-known result that "conditioning reduces entropy". $I(Z;X) = H(Z) - H(Z|X) \geq 0$, then $H(Z) \geq H(Z|X)$. Similarly, we have $H(Z) \geq H(Z|Y)$.
$X, Y$ are independent, so we can obtain $H(Z|X) = H(Y|X) = H(Y)$ and $H(Z|Y)=H(X|Y)=H(Y)$. Consequently, we have proved that $H(X+Y) \geq \max\{H(X), H(Y)\}$.

To prove the right-hand side, we need first prove that the entropy of a function of random variables is smaller than the joint entropy of the random variables.
Denote $Z=f(X,Y) = X + Y$, we have
\[
I((X,Y);Z) = H(Z) - H(Z|(X,Y)) = H(Z) = H(X,Y) - H((X,Y)|Z)\,,
\]
where the above equation is obtained with the fact that $H(Z|(X,Y)) = 0$.

From the definition of the entropy, we have $H((X,Y)|Z) \geq 0$. Combine the independence of $X,Y$, so we can get $H(Z) \leq H(X,Y) = H(X) + H(Y)$.
\end{proof}

\newpage
\section{MISSING PROOFS IN SECTION 4}

\subsection{Proof of Theorem~\ref{thm:sup_wasserstein}}\label{proof:sup_pop_bound}
\begin{theorem*}
$\forall (u,v) \in \Ucal\times\Vcal$, and $\forall \vect{\alpha} \in \Delta_N$, if Assumption~\ref{gen_assump} is satisfied, then 
\[
|R_{\Tcal}(u,v) - R_{\Scal^{\vect{\alpha}}}(u,v)| \leq  \mathbf{W}_1(\tilde{\Tcal}_u, \tilde{\Scal}^{\vect{\alpha}}_u) \leq \mathbf{W}_1(\Tcal, \Scal^{\vect{\alpha}})\,,
\]
where the first $\mathbf{W}_1$ distance is defined on the metric space $(\tilde{\Zcal}, \rho_{\tilde{z}})$, for $\rho_{\tilde{z}}(\tilde{z}, \tilde{z}^\prime) = \ell(y,y^\prime) + LM \rho_{\tilde{x}}(\tilde{x}, \tilde{x}^\prime)$, and the second one is defined on metric space $(\Zcal, \rho_z)$, for $\rho_{z}(z, z^\prime)= \ell(y,y^\prime) + LMK \rho_{x}(x, x^\prime)$.
\end{theorem*}

\begin{proof}
For any $\pi \in \Pi(\tilde{\Tcal}_u,\tilde{\Scal}^{\vect{\alpha}}_u)$, we have 
\[
\begin{aligned}
|R_{\Tcal}(u, v) - R_{\Scal^{\vect{\alpha}}}(u, v)| & = |\EE_{\tilde{Z}\sim \tilde{\Tcal}_u} \ell(h(v, \tilde{X}), Y) - \EE_{\tilde{Z}^\prime\sim \tilde{\Scal}^{\vect{\alpha}}_u} \ell(h(v, \tilde{X}^\prime), Y^\prime)| \\
&=  |\int_{\tilde{\Zcal}\times\tilde{\Zcal}} \ell(h(v, \tilde{x}), y) - \ell(h(v, \tilde{x}^\prime), y^\prime) d\pi(\tilde{z}, \tilde{z}^\prime)|\\
&\leq \int_{\tilde{Z}\times\tilde{Z}} |\ell(h(v, \tilde{x}), y) - \ell(h(v, \tilde{x}^\prime), y^\prime)|d\pi(\tilde{z}, \tilde{z}^\prime)\\
&= \int_{\tilde{Z}\times\tilde{Z}} |\ell(h(v, \tilde{x}), y) - \ell(h(v,\tilde{x}), y^\prime) + \ell(h(v, \tilde{x}), y^\prime) - \ell(h(v, \tilde{x}^\prime), y^\prime)|d\pi(\tilde{z}, \tilde{z}^\prime)\\
&\leq \int_{\tilde{Z}\times\tilde{Z}} |\ell(h(v, \tilde{x}), y) - \ell(h(v,\tilde{x}), y^\prime)| + |\ell(h(v, \tilde{x}), y^\prime) - \ell(h(v, \tilde{x}^\prime), y^\prime)|d\pi(\tilde{z}, \tilde{z}^\prime)\\
&\leq \int_{\tilde{Z}\times\tilde{Z}} \left(\ell(y, y^\prime) + M \rho_y(h(v,\tilde{x}), h(v,\tilde{x}^\prime)) \right)d\pi(\tilde{z}, \tilde{z}^\prime) \\
&\leq \int_{\tilde{Z}\times\tilde{Z}} \left(\ell(y, y^\prime) +  LM\rho_{\tilde{x}}(\tilde{x},\tilde{x}^\prime)\right) d\pi(\tilde{z}, \tilde{z}^\prime)
\end{aligned}
\]

The first two equalities hold with the definition of the population risk and Definition \ref{def:coupling} for coupling. The first inequality is obtained with Jensen's inequality for absolute function. And the second to last inequality is derived using the triangle inequality of the loss function and the Lipchitzness in Assumption~\ref{gen_assump}.

Let the metric on $\tilde{\Zcal}$ be $\rho_{\tilde{z}}(\tilde{z}, \tilde{z}^\prime) = \ell(y, y^\prime) + LM \rho_{\tilde{x}}(\tilde{x}, \tilde{x}^\prime)$.

\[
\begin{aligned}
\int_{\tilde{\Zcal}\times\tilde{\Zcal}} \rho_{\tilde{z}}(\tilde{z}, \tilde{z}^\prime) d\pi(\tilde{z}, \tilde{z}^\prime) 
&= \int_{\Zcal \times \Zcal} \left(\ell(y, y^\prime)  + LM \rho_{\tilde{x}}(\tilde{x}, \tilde{x}^\prime)\right) d\pi(\tilde{z}, \tilde{z}^\prime)\\
&= \int_{\Zcal \times \Zcal} \left(\ell(y, y^\prime)  + LM \rho_{\tilde{x}}(g(u,x), g(u, x^\prime))\right) d\gamma(z, z^\prime)\\
&\leq \int_{\Zcal \times \Zcal} \left(\ell(y, y^\prime) + LMK \rho_x(x , x^\prime)\right) d\gamma(z, z^\prime)\\
\end{aligned}
\]
Let the metric on $\Zcal$ be $\rho_{z}(z, z^\prime) = \ell(y, y^\prime) + LMK \rho_{x}(x, x^\prime)$. 

Let the optimal coupling on the representation space be $\pi_u^* = \argmin\limits_{\pi \in \Pi(\tilde{\Tcal}_u, \tilde{\Scal}_u^{\vect{\alpha}})} \int_{\tilde{\Zcal}\times\tilde{\Zcal}} \rho_{\tilde{z}}(\tilde{z}, \tilde{z}^\prime) d\pi(\tilde{z}, \tilde{z}^\prime)$ and the corresponding coupling on the original example space be $\gamma^*_u$. 

Let the optimal coupling on the original example space be $\gamma^* = \argmin\limits_{\gamma \in\Pi(\Tcal,\Scal^{\vect{\alpha}})} \int_{\Zcal\times\Zcal} \rho_z(z, z^\prime) d\gamma(z, z^\prime)$ and the corresponding coupling on the representation space be $\pi_u$.
So we have:
\[
\begin{aligned}
|R_{\Tcal}(u,v) - R_{\Scal^{\vect{\alpha}}}(u,v)| &\leq \int_{\tilde{\Zcal}\times\tilde{\Zcal}} \rho_{\tilde{z}}(\tilde{z}, \tilde{z}^\prime) d\pi_u^*(\tilde{z}, \tilde{z}^\prime) = \mathbf{W}_1(\tilde{\Tcal}_u, \tilde{\Scal}^{\vect{\alpha}}_u)\\
&\leq \int_{\tilde{\Zcal}\times\tilde{\Zcal}} \rho_{\tilde{z}}(\tilde{z}, \tilde{z}^\prime) d\pi_u(\tilde{z}, \tilde{z}^\prime) \leq \int_{\Zcal\times\Zcal} \rho_z(z, z^\prime) d\gamma^*(z, z^\prime) = \mathbf{W}_1(\Tcal, \Scal^{\vect{\alpha}})
\end{aligned}
\]

The last inequality is obtained since $\pi_u\neq\pi_u^*$ may not be the optimal coupling on the representation space.

\end{proof}

\newpage
\subsection{Proof of Theorem~\ref{gen_supervised}}\label{proof:sup_gen_bound}
\begin{theorem*}
If Assumption~\ref{gen_assump} and \ref{assum_gauss} are satisfied, then we can bound the generalization gap of the supervised multi-source domain adaptation algorithm $\Acal$ for any $\vect{\alpha}\in \Delta_N$ and $0 \leq \epsilon \leq 1$ with: 
\[
\begin{aligned}
gen(\Tcal, \Acal) &\leq \sigma\sqrt{2(\frac{(1-\epsilon)^2}{m_t} + \sum_{i=1}^N\frac{\epsilon^2\alpha_i^2}{m_i}) I(U,V;S^{\vect{\alpha}}, T)} + \sigma^\prime\sqrt{2\epsilon^2(\sum_{i=1}^N\frac{\alpha_i^2}{m_i} + \frac{1}{m_t})I(U;S^{\vect{\alpha}},T)}\, .
\end{aligned}
\]
\end{theorem*}

\begin{proof}

\[
\begin{aligned}
gen(\Tcal, \Acal) &\leq \EE_{U,V,S^{\vect{\alpha}}, T} [R^{\epsilon}_{\Scal^{\vect{\alpha}}, \Tcal}(U, V) - \hat{R}^{\epsilon}_{\Scal^{\vect{\alpha}}, \Tcal}(U, V)] \\
&=  \EE_{U,V,S^{\vect{\alpha}}, T} [(1-\epsilon) (R_{\Tcal}(U,V)-\hat{R}_{\Tcal}(U,V)) + \epsilon (R_{\Scal^{\vect{\alpha}}}(U, V) - \hat{R}_{\Scal^{\vect{\alpha}}}(U,V))] \\
&\quad\quad + \epsilon \EE_{U,V,S^{\vect{\alpha}}, T} [\mathbf{W}_1(\tilde{\Tcal}_U, \tilde{\Scal}^{\vect{\alpha}}_U) -\hat{\mathbf{W}}_1(\tilde{\Tcal}_U, \tilde{\Scal}^{\vect{\alpha}}_U)]\\
\end{aligned}
\]

\paragraph{Step 1 Bounding the Empirical Risks}
\[
\begin{aligned}
B_r = \EE_{U,V,S^{\vect{\alpha}}, T} \left[(1 - \epsilon) [R_{\Tcal}(U,V)-\hat{R}_{\Tcal}(U,V)] + \epsilon [R_{\Scal^{\vect{\alpha}}}(U, V) - \hat{R}_{\Scal^{\vect{\alpha}}}(U,V)] \right]
\end{aligned}
\]
$(U,V)\in \Ucal\times\Vcal$ are random variables outputs by the supervised MDA algorithm, which depend on the input datasets $S_{1:N}$, $T$ and a fixed $\vect{\alpha}$ (not random). Let $(\tilde{U},\tilde{V})$ be an independent copy of $(U,V)$ such that $(\tilde{U},\tilde{V})\indep (S_{1:N}, T)$. 
Then let us define 
\[
\begin{aligned}
f(U,V, S^{\vect{\alpha}}, T) &\eqdef (1-\epsilon) \hat{R}_{\Tcal}(U, V) + \epsilon \hat{R}_{S^{\vect{\alpha}}}(U, V)\\
&= (1-\epsilon)\frac{1}{m_t} \sum_{j=1}^{m_t} \ell(h(V,g(U,X_j^t)),Y_j^t) + \epsilon \sum_{i=1}^N \frac{\alpha_i}{m_i} \sum_{j=1}^{m_i}\ell(h(V, g(U, X_{i,j}^s)), Y_{i,j}^s)
\end{aligned}
\]

Since the loss function is $\sigma$-sub-Gaussian (Assumption~\ref{assum_gauss}),  applying Lemma~\ref{lemma_sub_gaussian}, we can obtain that $f(\tilde{U}, \tilde{V}, S^{\vect{\alpha}},T)$ is $\sqrt{\frac{(1-\epsilon)^2}{m_t} + \epsilon^2\sum_{i=1}^N\frac{\alpha_i^2}{m_i}}\sigma$-sub-Gaussian.
\[
\begin{aligned}
I(U,V;S^{\vect{\alpha}},T) &= \KL(P_{U,V,S^{\vect{\alpha}}, T}\|P_{U,V} P_{S^{\vect{\alpha}},T})\\
&=\sup_{\mathfrak{g}}\{\EE_{U,V,S^{\vect{\alpha}},T} \mathfrak{g}(U,V,S^{\vect{\alpha}}, T) - \log\EE_{\tilde{U}, \tilde{V},S^{\vect{\alpha}}, T}[\exp^{\mathfrak{g}(\tilde{U}, \tilde{V},S^{\vect{\alpha}}, T)}]\}\\
&\geq \lambda \EE_{U,V,S^{\vect{\alpha}}, T}[f(U,V,S^{\vect{\alpha}}, T)] -\lambda\EE_{\tilde{U}, \tilde{V},S^{\vect{\alpha}}, T}[f(\tilde{U}, \tilde{V},S^{\vect{\alpha}}, T)] - \psi_{\tilde{U}, \tilde{V},S^{\vect{\alpha}}, T}(\lambda), \forall \lambda\in \Reals\\
&\geq \lambda \EE_{U,V,S^{\vect{\alpha}}, T}[f(U,V,S^{\vect{\alpha}}, T)] -\lambda\EE_{\tilde{U}, \tilde{V},S^{\vect{\alpha}}, T}[f(\tilde{U}, \tilde{V},S^{\vect{\alpha}}, T)] - \frac{\lambda^2\sigma^2(\frac{(1-\epsilon)^2}{m_t} + \epsilon^2\sum_{i=1}^N\frac{\alpha_i^2}{m_i})}{2}
\end{aligned}
\]

We have:
\[
\begin{aligned}
\EE_{\tilde{U}, \tilde{V},S^{\vect{\alpha}}, T}[f(\tilde{U}, \tilde{V},S^{\vect{\alpha}}, T)]
&= \epsilon\EE_{\tilde{U}, \tilde{V},S^{\vect{\alpha}}} \hat{R}_{\Scal^{\vect{\alpha}}}(\tilde{U},\tilde{V})
+ (1 - \epsilon) \EE_{\tilde{U}, \tilde{V}, T} \hat{R}_{\Tcal}(\tilde{U}, \tilde{V})\\
&= \epsilon\EE_{\tilde{U}, \tilde{V}} R_{\Scal^{\vect{\alpha}}}(\tilde{U},\tilde{V})
+ (1 - \epsilon) \EE_{\tilde{U}, \tilde{V}} R_{\Tcal}(\tilde{U},\tilde{V})\\
&= \epsilon\EE_{U,V,S^{\vect{\alpha}}, T} R_{\Scal^{\vect{\alpha}}}(U,V)
+ (1 - \epsilon) \EE_{U,V, S^{\vect{\alpha}}, T} R_{\Tcal}(U,V)\
\end{aligned}
\]

Place the above term into the inequality, we get  $-\lambda B_r - \frac{\lambda^2\sigma^2(\frac{(1-\epsilon)^2}{m_t} + \epsilon^2\sum_{i=1}^N\frac{\alpha_i^2}{m_i})}{2} \leq I(U,V;S^{\vect{\alpha}}, T)$

Consequently, we have $|B_r| \leq \sqrt{2(\sigma^2(\frac{(1-\epsilon)^2}{m_t} + \epsilon^2\sum_{i=1}^N\frac{\alpha_i^2}{m_i})) I(U,V;S^{\vect{\alpha}}, T)}$.

\paragraph{Step 2 Bounding the Wasserstein Distance}
\[B_w = \EE_{U,V,S^{\vect{\alpha}}, T}\left[\epsilon \mathbf{W}_1(\tilde{\Tcal}_U, \tilde{\Scal}^{\vect{\alpha}}_U)  - \epsilon \hat{\mathbf{W}}_1(\tilde{\Tcal}_U, \tilde{\Scal}^{\vect{\alpha}}_U)\right] = \EE_{U,S^{\vect{\alpha}}, T}\left[\epsilon \mathbf{W}_1(\tilde{\Tcal}_U, \tilde{\Scal}^{\vect{\alpha}}_U)  - \epsilon \hat{\mathbf{W}}_1(\tilde{\Tcal}_U, \tilde{\Scal}^{\vect{\alpha}}_U)\right] \]

Let
\[
\begin{aligned}
f^\prime(U,S^{\vect{\alpha}},T) =  \hat{\mathbf{W}}_1(\tilde{\Tcal}_U, \tilde{\Scal}^{\vect{\alpha}}_U) = \sup_{v^\prime} \left[\frac{1}{m_t}\sum_{j=1}^{m_t} \tilde{f}(v^\prime, (g(U,X_j^t), Y_j^t)) - \sum_{i=1}^N \frac{\alpha_i}{m_i} \sum_{j=1}^{m_i} \tilde{f}(v^\prime, (g(U, X_{i,j}^s), Y_{i,j}^s)) \right]\,,
\end{aligned}
\]

\[
\begin{aligned}
I(U;S^{\vect{\alpha}},T) &= \KL(P_{U,S^{\vect{\alpha}}, T}\|P_{U} P_{S^{\vect{\alpha}},T})\\
&=\sup_{\mathfrak{g}}\{\EE_{U,S^{\vect{\alpha}},T} \mathfrak{g}(U,S^{\vect{\alpha}}, T) - \log\EE_{\tilde{U},S^{\vect{\alpha}}, T}[\exp^{\mathfrak{g}(\tilde{U},S^{\vect{\alpha}}, T)}]\}\\
&\geq \lambda \EE_{U,S^{\vect{\alpha}}, T}[f^\prime(U,S^{\vect{\alpha}}, T)] -\lambda\EE_{\tilde{U},S^{\vect{\alpha}}, T}[f^\prime(\tilde{U},S^{\vect{\alpha}}, T)] - \psi_{\tilde{U},S^{\vect{\alpha}}, T}(\lambda), \forall \lambda\in \Reals\\
\end{aligned}
\]

$\Vcal^\prime$ is assumed to be the subset that does not affect the supremum. So we have:
\[
\begin{aligned}
\EE_{\tilde{U},S^{\vect{\alpha}}, T}[f^\prime(\tilde{U},S^{\vect{\alpha}}, T)] &= \EE_{\tilde{U}} \sup_{v^\prime} \EE_{S^{\vect{\alpha}}, T} \left[\frac{1}{m_t}\sum_{j=1}^{m_t} \tilde{f}(v^\prime, (g(\tilde{U},X_j^t), Y_j^t)) - \sum_{i=1}^N \frac{\alpha_i}{m_i} \sum_{j=1}^{m_i} \tilde{f}(v^\prime, (g(\tilde{U}, X_{i,j}^s), Y_{i,j}^s)) \right]\\
&= \EE_{\tilde{U}} \mathbf{W}_1(\tilde{\Tcal}_{\tilde{U}}, \tilde{\Scal}^{\vect{\alpha}}_{\tilde{U}}) = \EE_{U, S^{\vect{\alpha}}, T} \mathbf{W}_1(\tilde{\Tcal}_{U}, \tilde{\Scal}^{\vect{\alpha}}_{U})
\end{aligned}
\]

From Assumption \ref{assum_gauss}, $\tilde{f}$ is $\sigma^\prime$-sub-Gaussian for any $u$ and $v^\prime$, so the supremum does not affect the sub-Gaussianity. Apply Lemma~\ref{lemma_sub_gaussian}, we know that $\hat{\mathbf{W}}_1(\tilde{\Tcal}_u, \tilde{\Scal}^{\vect{\alpha}}_u)$ is $\sqrt{\frac{1}{m_t} + \sum_{i=1}^N \frac{\alpha_i^2}{m_i}}\sigma^\prime$-sub-Gaussian.

With the same proof process as bounding the empirical risks, we obtain that:
\[
|B_w| \leq \sqrt{2{\sigma^\prime}^2\epsilon^2(\frac{1}{m_t} + \sum_{i=1}^N \frac{\alpha_i^2}{m_i}) I(U;S^{\vect{\alpha}}, T)}
\]

Finally, we have:
\[
\begin{aligned}
gen(\Tcal,\Acal) &\leq B_r + B_w\\
&\leq  \sqrt{2(\sigma^2(\frac{(1-\epsilon)^2}{m_t} + \epsilon^2\sum_{i=1}^N\frac{\alpha_i^2}{m_i})) I(U,V;S^{\vect{\alpha}}, T)} + \sqrt{2{\sigma^\prime}^2\epsilon^2(\frac{1}{m_t} + \sum_{i=1}^N \frac{\alpha_i^2}{m_i}) I(U;S^{\vect{\alpha}}, T)}
\end{aligned}
\]

Conclude the proof.

\end{proof}

\newpage
\section{MISSING PROOFS IN SECTION 5}

\subsection{Proof of Theorem~\ref{thm:pseudo_label}}\label{proof:unsup_pop_bound}

\begin{definition}(Joint Approximation Error) Let $\tilde{x}=g(u,x), \tilde{x}^\prime=g(u,x^\prime)$,
then let the \textbf{optimal} coupling on $\Zcal=\Xcal\times\Ycal$ given the pseudo labeling function $h(v,g(u,x))$ be
\[
\gamma_{u,v}^* \eqdef \argmin\limits_{\gamma \in \Pi(\Tcal_{u,v}, \Scal^{\vect{\alpha}})}\int_{\Zcal\times\Zcal} \left[ LM\rho_{\tilde{x}}(\tilde{x}, \tilde{x}^\prime) + \ell(\hat{y}, y^\prime) \right] d\gamma(\hat{z},z^\prime)
\]
Then the joint approximation error is
\[
\begin{aligned}
R^*_{rep}(u,v) &\eqdef LM \int_{\Zcal\times\Zcal} \left[ \rho_{\tilde{x}}(g(u^*,x) ,g(u^*,x^\prime)) -\rho_{\tilde{x}}(g(u,x), g(u,x^\prime))\right] d\gamma_{u,v}^*(\hat{z}, z^\prime)\,,
\end{aligned}
\]
where $R_{rep}^*(u,v)$ characterizes the approximation error for the domain shift on the representation space of $u^*$ (the representation counterpart of the ideal joint hypothesis).
\end{definition}\label{def:joint_approx}
\begin{theorem*}
For any given $\vect{\alpha}\in\Delta_N$ and $\forall u,v \in \Ucal \times \Vcal$, if Assumption~\ref{gen_assump} is satisfied, 
then the target population risk can be bounded by: 
\[
R_{\Tcal}(u,v) \leq \mathbf{W}_1(\tilde{\Tcal}_{u,v}, \tilde{\Scal}_u^{\vect{\alpha}}) + R^*_{rep}(u,v) + R^*\,.
\]
\end{theorem*}
\begin{proof}
Let $u^*,v^* = \argmin_{u,v} R_{\Scal^{\vect{\alpha}}}(u,v) + R_{\Tcal}(u,v)$ be the ideal joint hypothesis, we have:
\[
\begin{aligned}
R_{\Tcal}(u,v) &= \EE_{Z\sim\Tcal}\ell(h(v,g(u,X)), Y)\\
&\leq \EE_{Z\sim\Tcal} \ell(h(v, g(u,X)), h(v^*, g(u^*,X))) + \EE_{Z\sim\Tcal} \ell(h(v^*, g(u^*,X)), Y)\\
&= \EE_{\hat{Z}\sim\Tcal_{u,v}} \ell(h(v^*, g(u^*,X)), \hat{Y}) + R_{\Tcal}(u^*, v^*)\\
&= \EE_{Z\sim\Tcal_{u,v}} \ell(h(v^*, g(u^*,X)), Y) - \EE_{Z^\prime\sim\Scal^{\vect{\alpha}}} \ell(h(v^*, g(u^*,X^\prime)), Y^\prime)\\
&\quad\quad + \EE_{Z^\prime\sim\Scal^{\vect{\alpha}}} \ell(h(v^*, g(u^*,X^\prime)), Y^\prime) + R_{\Tcal}(u^*, v^*)\\
&\leq |\EE_{\hat{Z}\sim\Tcal_{u,v}} \ell(h(v^*, g(u^*,X)), \hat{Y}) - \EE_{Z^\prime\sim\Scal^{\vect{\alpha}}} \ell(h(v^*, g(u^*,X^\prime)), Y^\prime)|\\
&\quad\quad + R_{\Scal^{\vect{\alpha}}}(u^*,v^*) + R_{\Tcal}(u^*, v^*)\,.
\end{aligned}
\]
For alignment on representation space, we have $\forall \gamma \in \Pi(\Tcal_{u,v},\Scal^{\vect{\alpha}})$ that:

\[
\begin{aligned}
|\EE_{Z\sim\Tcal_{u,v}} &\ell(h(v^*, g(u^*,X)), \hat{Y}) - \EE_{Z^\prime\sim\Scal^{\vect{\alpha}}} \ell(h(v^*, g(u^*,X^\prime)), Y^\prime)|\\
&\leq \int_{\Zcal\times\Zcal} |\ell(h(v^*, g(u^*,x)), \hat{y}) - \ell(h(v^*, g(u^*,x^\prime)), y^\prime)| d\gamma(\hat{z}, z^\prime)\\
&\leq \int_{\Zcal\times\Zcal} 
 |\ell(h(v^*, g(u^*,x)), y^\prime) - \ell(h(v^*, g(u^*,x^\prime)), y^\prime)| \\
&\quad\quad\quad + |\ell(h(v^*, g(u^*,x)), \hat{y}) - \ell(h(v^*, g(u^*,x)), y^\prime)|d\gamma(\hat{z}, z^\prime)\\
& \leq \int_{\Zcal\times\Zcal} 
 LM\rho_{\tilde{x}}(g(u^*,x), g(u^*,x^\prime)) + \ell(\hat{y}, y^\prime)d \gamma(\hat{z},z^\prime)\\
&= \int_{\Zcal\times\Zcal} 
 LM\rho_{\tilde{x}}(g(u^*,x), g(u^*,x^\prime)) - LM\rho_{\tilde{x}}(g(u,x), g(u,x^\prime)) d \gamma(\hat{z},z^\prime) \\
&\quad\quad + \int_{\Zcal\times\Zcal} 
 LM\rho_{\tilde{x}}(g(u,x), g(u,x^\prime)) + \ell(\hat{y}, y^\prime)d \gamma(\hat{z},z^\prime)\\
\end{aligned}
\]

Let the metric on $\tilde{\Zcal}$ be $\rho_{\tilde{z}}(\hat{\tilde{z}}, \tilde{z}^\prime) = LM \rho_{\tilde{x}}(\tilde{x}, \tilde{x}^\prime) + \ell(\hat{y}, y^\prime)$. 
Let the optimal coupling on the transformed joint space $\tilde{\Zcal}$ be $\pi_{u,v}^* = \argmin\limits_{\pi \in \Pi(\tilde{\Tcal}_{u,v}, \tilde{\Scal}_u^{\vect{\alpha}})} \int_{\tilde{\Zcal}\times\tilde{\Zcal}}\rho_{\tilde{z}}(\hat{\tilde{z}}, \tilde{z}^\prime)d\pi(\hat{\tilde{z}},\tilde{z}^\prime)$, so the corresponding coupling on the original example space is: 
$$\gamma^*_{u,v} = \argmin\limits_{\gamma \in \Pi(\Tcal_{u,v}, \Scal^{\vect{\alpha}})}\int_{\Zcal\times\Zcal} LM \rho_{\tilde{x}}(g(u,x), g(u,x^\prime)) + \ell(h(v, g(u,x^\prime)), y^\prime) d\gamma(\hat{z},z^\prime)\,.$$

Hence, $$R_{\Tcal}(u,v) \leq \mathbf{W}_1(\tilde{\Tcal}_{u,v}, \tilde{\Scal}_u^{\vect{\alpha}}) + R^*_{rep}(u,v) + R^*\,.
$$

where $R^*_{rep}(u,v) = \int_{\Zcal\times\Zcal} 
 LM\rho_{\tilde{x}}(g(u^*,x), g(u^*,x^\prime)) - LM\rho_{\tilde{x}}(g(u,x), g(u,x^\prime)) d \gamma_{u,v}^*(\hat{z},z^\prime)$ is the joint approximation error (Definition~\ref{def:joint_approx}), which is controlled by the representation learning function $g(u,x)$ and the optimal coupling $\gamma_{u,v}^*(\hat{z}, z^\prime)$.
\end{proof}

\subsection{Proof of Theorem~\ref{gen_un_pseudo}} \label{proof:unsup_gen_bound}
\begin{theorem*}
If Assumption~\ref{gen_assump} and \ref{assum_unsup} are satisfied, then we have the following bound on the expected generalization gap of the unsupervised multi-source domain adaptation algorithm with pseudo label $\Acal^{un}$ for any $\vect{\alpha}\in\Delta_N$:
\[
\begin{aligned}
gen(\Tcal,\Acal^{un}) \leq  \sqrt{2{\sigma^\prime}^2(\sum_{i=1}^N\frac{\alpha_i^2}{m_i} + \frac{1}{m_t^\prime})I(U,V;S^{\vect{\alpha}}, T_X^\prime)} + \EE_{U,V} R^*_{rep}(U,V) + R^*
\end{aligned}
\]
\end{theorem*}
\begin{proof}
By definition:
\[
gen(\Tcal,\Acal^{un}) \eqdef \EE_{U,V,S^{\vect{\alpha}},T_X^\prime} [R_{\Tcal}(U,V) - \hat{\mathbf{W}}_1(\tilde{\Tcal}_{U,V}, \tilde{\Scal}_U^{\vect{\alpha}})]
\]
Combine Theorem~\ref{thm:pseudo_label}, we have
\[
\begin{aligned}
gen(&\Tcal,\Acal^{un})\leq \EE_{U,V,S^{\vect{\alpha}},T_X^\prime} [\mathbf{W}_1(\tilde{\Tcal}_{U,V}, \tilde{\Scal}_U^{\vect{\alpha}}) - \hat{\mathbf{W}}_1(\tilde{\Tcal}_{U,V}, \tilde{\Scal}_U^{\vect{\alpha}})] 
+ \EE_{U,V}R^*_{rep}(U,V) + R^*
\end{aligned}
\]
Let 
\[
\begin{aligned}
f^\prime(U,V, S^{\vect{\alpha}}, T_X^\prime) &= \hat{W}_1(\tilde{\Tcal}_{U,V}, \tilde{\Scal}_U^{\vect{\alpha}}) \\
&= \sup_{v^\prime} \left[\frac{1}{m_t^\prime} \sum_{j=1}^{m_t^\prime} \tilde{f}(v^\prime, (g(U,X_j^t), h(V, g(U, X_j^t)))) - \sum_{i=1}^N \frac{\alpha_i}{m_i} \sum_{j=1}^{m_i} \tilde{f}(v^\prime, (g(U, X_{i,j}^s), Y_{i,j}^s)) \right]
\end{aligned}
\]

\[
\begin{aligned}
I(U,V; S^{\vect{\alpha}},T_X^\prime) &= \KL(P_{U,V,S^{\vect{\alpha}}, T}\|P_{U,V} P_{S^{\vect{\alpha}},T_X^\prime})\\
&=\sup_{\mathfrak{g}}\{\EE_{U,V,S^{\vect{\alpha}},T_X^\prime} \mathfrak{g}(U,V,S^{\vect{\alpha}}, T_X^\prime) - \log\EE_{\tilde{U},\tilde{V}, S^{\vect{\alpha}}, T_X^\prime}[\exp^{\mathfrak{g}(\tilde{U},\tilde{V}, S^{\vect{\alpha}}, T_X^\prime)}]\}\\
&\geq \lambda \EE_{U,V,S^{\vect{\alpha}}, T_X^\prime}[f(U,V, S^{\vect{\alpha}}, T_X^\prime)] -\lambda\EE_{\tilde{U},\tilde{V}, S^{\vect{\alpha}}, T_X^\prime}[f(\tilde{U},\tilde{V}, S^{\vect{\alpha}}, T_X^\prime)] - \psi_{\tilde{U},\tilde{V}, S^{\vect{\alpha}}, T_X^\prime}(\lambda), \forall \lambda\in \Reals\\
\end{aligned}
\]
$\Vcal^\prime$ is assumed to be the set that does not affect the supremum, so we have:
\[
\begin{aligned}
\EE_{\tilde{U}, \tilde{V},S^{\vect{\alpha}}, T_X^\prime}[f^\prime(\tilde{U}, \tilde{V},S^{\vect{\alpha}}, T_X^\prime)]
&= \EE_{\tilde{U}, \tilde{V}} \sup_{v^\prime} \EE_{\hat{Z}\sim \Tcal_{\tilde{U},\tilde{V}}} \tilde{f}(v^\prime, (g(\tilde{U}, X), \hat{Y})) - \EE_{Z^\prime \sim \Scal^{\vect{\alpha}}} \tilde{f}(v^\prime, (g(\tilde{U}, X^\prime), Y^\prime))\\
&= \EE_{\tilde{U}, \tilde{V}}\mathbf{W}_1(\tilde{\Tcal}_{\tilde{U},\tilde{V}}, \tilde{\Scal}^{\vect{\alpha}}_{\tilde{U}}) = \EE_{U,V,S^{\vect{\alpha}}, T_X^\prime}\mathbf{W}_1(\tilde{\Tcal}_{U, V}, \tilde{\Scal}^{\vect{\alpha}}_{U})
\end{aligned}
\]

Then from Assumption~\ref{assum_unsup}, $\tilde{f}$ is $\sigma^\prime$-sub-Gaussian for any $u,v,v^\prime$ under $\hat{Z}\sim\Tcal_{u,v}$, and $\sigma^\prime$-sub-Gaussian for any $u,v^\prime$ under $Z\sim\Scal_i,\forall i\in[N]$. Apply Lemma~\ref{lemma_sub_gaussian}, we have $\hat{W}_1(\tilde{\Tcal}_{U,V}, \tilde{\Scal}_U^{\vect{\alpha}})$ is $\sqrt{\frac{1}{m_t^\prime} + \sum_{i=1}^N\frac{\alpha_i^2}{m_i}}\sigma^\prime$-sub-Gaussian. 
So we can obtain 
$-\lambda\EE_{U,V,S^{\vect{\alpha}},T_X^\prime}\left(\mathbf{W}_1(\tilde{\Tcal}_{U, V}, \tilde{\Scal}^{\vect{\alpha}}_{U}) - \hat{\mathbf{W}}_1(\tilde{\Tcal}_{U, V}, \tilde{\Scal}^{\vect{\alpha}}_{U})\right)   - \lambda^2{\sigma^\prime}^2(\frac{  1}{m_t^\prime} + \sum_{i=1}^N\frac{\alpha_i^2}{m_i})/2 \leq I(U,V;S^{\vect{\alpha}}, T_X^\prime)$.

Consequently, we have:
\[\lvert\EE_{U,V,S^{\vect{\alpha}},T_X^\prime}\left(\mathbf{W}_1(\tilde{\Tcal}_{U, V}, \tilde{\Scal}^{\vect{\alpha}}_{U}) - \hat{\mathbf{W}}_1(\tilde{\Tcal}_{U, V}, \tilde{\Scal}^{\vect{\alpha}}_{U})\right)\rvert \leq \sqrt{2{\sigma^\prime}^2(\frac{1}{m_t^\prime} + \sum_{i=1}^N\frac{\alpha_i^2}{m_i}) I(U,V;S^{\vect{\alpha}}, T_X^\prime)}\]
Thus,
\[
\begin{aligned}
gen(&\Tcal,\Acal^{un})\leq \EE_{U,V,S^{\vect{\alpha}},T_X^\prime} [\mathbf{W}_1(\tilde{\Tcal}_{U,V}, \tilde{\Scal}_U^{\vect{\alpha}}) - \hat{\mathbf{W}}_1(\tilde{\Tcal}_{U,V}, \tilde{\Scal}_U^{\vect{\alpha}})] 
+ \EE_{U,V}R^*_{rep}(U,V) + R^*\\
&\leq \sqrt{2{\sigma^\prime}^2(\sum_{i=1}^N\frac{\alpha_i^2}{m_i} + \frac{1}{m_t^\prime})I(U,V;S^{\vect{\alpha}}, T_X^\prime)} + \EE_{U,V} R^*_{rep}(U,V) + R^*
\end{aligned}
\]
Conclude the proof.
\end{proof}

\newpage
\section{MISSING PROOFS IN SECTION 6}

\subsection{Proof of Theorem~\ref{thm:gen_gradient_norm}} \label{proof:gradient_bound}
\begin{theorem*}
Following Theorem~\ref{thm:sup_wasserstein},~\ref{thm:pseudo_label} ~\ref{gen_supervised},~\ref{gen_un_pseudo}, define a weight parameter $\tau$ that balances the unsupervised and supervised MDA. Adopt the aforementioned choice of $\tilde{\Fcal}$ with $\sigma^\prime=\sigma$ and the SGLD updates described above. Then we can obtain the gradient norm bound for the expected target risk, $\forall \vect{\alpha}\in\Delta_N, \tau,\epsilon \in [0,1]$, we have:
\[
\begin{aligned}
\EE_{U,V,S^{\vect{\alpha}}, T, T_X^\prime} R_{\Tcal}(U, V) &\leq \EE_{U,V, S^{\vect{\alpha}}, T, T_X^\prime} \hat{R}_{\epsilon,\tau}^{\vect{\alpha}}(U,V) + \tau \sigma\sqrt{2(\frac{(1-\epsilon)^2}{m_t} + \epsilon^2\sum_{i=1}^N\frac{\alpha_i^2}{m_i})(\delta_u + \delta_v)}\\
&+ \tau \epsilon \sigma \sqrt{2(\sum_{i=1}^N\frac{\alpha_i^2}{m_i} + \frac{1}{m_t}) \delta_u} + (1-\tau)\EE_{U,V} R_{rep}^*(U,V)\\
&+ (1-\tau)(\sigma\sqrt{2(\sum_{i=1}^N\frac{\alpha_i^2}{m_i} + \frac{1}{m_t^\prime})(\delta_u + \delta_v)} + R^*)\,,
\end{aligned}
\]
$\delta_{u}\eqdef\sum_{k=1}^K \frac{(\eta_u^k)^2\EE\|G_u^k\|_2^2 }{2\sigma_k^2}$ and $\delta_v\eqdef\sum_{k=1}^K \frac{(\eta_v^k)^2\EE\|G_v^k\|_2^2}{2\sigma_k^2}$ are the accumulated gradient  norm for $U$ and $V$, respectively.
\end{theorem*}
\begin{proof}
Combine Theorem 4.1, 4.2, 5.1 and 5.2 we have:
\[
\begin{aligned}
\EE_{U,V, S^{\vect{\alpha}}, T, T_X} R_{\Tcal}(U, V) &\leq 
 \EE_{U,V, S^{\vect{\alpha}}, T, T_X} \hat{R}_{\epsilon,\tau}^{\vect{\alpha}}(U,V)  + \tau \sqrt{2\epsilon^2{\sigma}^2(\sum_{i=1}^N\frac{\alpha_i^2}{m_i} + \frac{1}{m_t})I(U;S^{\vect{\alpha}},T)} \\
 &+ \tau \sqrt{2\sigma^2(\frac{(1-\epsilon)^2}{m_t} + \epsilon^2\sum_{i=1}^N\frac{\alpha_i^2}{m_i}) I(U,V;S^{\vect{\alpha}}, T)} + (1-\tau)(R^* +\EE_{U,V}R_{rep}^*(U,V))\\
&+ (1-\tau)\sqrt{2{\sigma}^2(\sum_{i=1}^N\frac{\alpha_i^2}{m_i} + \frac{1}{m_t^\prime})I(U,V;S^{\vect{\alpha}}, T_X^\prime)}\,,
\end{aligned}
\]

We use SGLD to optimize the empirical risk:
\[
\begin{aligned}
\hat{R}_{\epsilon,\tau}^{\vect{\alpha}}(u,v) &\eqdef \tau(1-\epsilon)\hat{R}_{\Tcal}(u,v) + \tau\epsilon\hat{R}_{\Scal^{\vect{\alpha}}}(u,v)\\
& + \tau\epsilon\hat{\mathbf{W}}_1(\tilde{\Tcal}_u, \tilde{\Scal}^{\vect{\alpha}}_u)) + (1-\tau)\hat{\mathbf{W}}_1(\tilde{\Tcal}_{u,v}, \tilde{\Scal}_u^{\vect{\alpha}})
\end{aligned}
\]

Since we jointly minimize the objective function w.r.t $u,v$ and maximize w.r.t $v^\prime$ by definition of the Wasserstein distances, we analyze $U,V$ given a fixed $v^\prime$ updated from previous steps.  
Given the batch size $|B|$, let the batch estimation of the empirical risks be:

$\hat{R}_{T_{B}^k}(U,V) = \frac{1}{|B|}\sum_{j=1}^{|B|} \ell(h(V, g(U,X_{j}^t)), Y_j^t)$

$\hat{R}_{S_{B_{1:N}}^k}^{\vect{\alpha}}(U,V) =  \sum_{i=1}^N \frac{\alpha_i}{|B|} \sum_{j=1}^{|B|} \ell(h(V, (g(U, X_{i,j}^s)), Y_{i,j}^s))$

$\hat{R}_{T_{X_B}^k}(U,V, v^\prime) = \frac{1}{|B|}\sum_{j=1}^{|B|} \ell(h(v^\prime, g(U, X_j^t)), h(V,g(U,X_j^t)))$

$\hat{\mathbf{W}}_1^k(\tilde{\Tcal}_U, \tilde{\Scal}^{\vect{\alpha}}_U)) = \hat{R}_{T_{B}^k}(U,v^\prime) - \hat{R}_{S_{B_{1:N}}^k}^{\vect{\alpha}}(U,v^\prime)$

$\hat{\mathbf{W}}_1^k(\tilde{\Tcal}_{U,V}, \tilde{\Scal}_U^{\vect{\alpha}})= \hat{R}_{T_{X_B}^k}(U,V, v^\prime) - \hat{R}_{S_{B_{1:N}}^k}^{\vect{\alpha}}(U,v^\prime)$

For updating $U$, note the gradients w.r.t three data batches as:

$G_u(U_{k-1},V_{k-1}, T_{B}^k) = \tau(1-\epsilon)\nabla_U\hat{R}_{T_{B}^k}(U_{k-1},V_{k-1}) + \tau\epsilon \nabla_U \hat{R}_{T_{B}^k}(U,v^\prime)$

$G_u(U_{k-1}, V_{k-1}, T_{X_B}^k) = (1-\tau) \nabla_U \hat{R}_{T_{X_B}^k}(U_{k-1},V_{k-1}, v^\prime)$

$G_u(U_{k-1}, V_{k-1}, S_{B_{1:N}}^k) = \tau\epsilon \nabla_U \hat{R}_{S_{B_{1:N}}^k}^{\vect{\alpha}}(U_{k-1},V_{k-1}) -(1-\tau + \epsilon\tau)\nabla_U \hat{R}_{S_{B_{1:N}}^k}^{\vect{\alpha}}(U_{k-1},v^\prime)$

Similarly, for updating $V$, we have:

$G_v(U_{k-1}, V_{k-1}, T_{B}^k) = \tau(1-\epsilon)\nabla_V\hat{R}_{T_{B}^k}(U_{k-1},V_{k-1})$

$G_v(U_{k-1}, V_{k-1}, T_{X_B}^k) = (1-\tau) \nabla_V \hat{R}_{T_{X_B}^k}(U_{k-1},V_{k-1}, v^\prime)$

$G_v(U_{k-1}, V_{k-1}, S_{B_{1:N}}^k) = \tau\epsilon \nabla_V \hat{R}_{S_{B_{1:N}}^k}^{\vect{\alpha}}(U_{k-1},V_{k-1})$

Denote the overall gradient updates for $U$ and $V$ respectively as:
$$G_u^k = G_u(U_{k-1},V_{k-1}, T_{B}^k) + G_u(U_{k-1}, V_{k-1}, T_{X_B}^k) + G_u(U_{k-1}, V_{k-1}, S_{B_{1:N}}^k)$$

$$G_v^k = G_v(U_{k-1},V_{k-1}, T_{B}^k) + G_v(U_{k-1}, V_{k-1}, T_{X_B}^k) + G_v(U_{k-1},V_{k-1}, S_{B_{1:N}}^k)$$

Then the updates with noise injected:

\[
\begin{aligned}
U_k &= U_{k-1} - \eta_k^u G_u^k + \xi_k^u\\
V_{k} &= V_{k-1} - \eta_k^v G_v^k + \xi_k^v\,,
\end{aligned}
\]
where $\xi_k^u \sim N(0, \sigma_k^2 I_{d_u})$ and $\xi_k^v \sim N(0, \sigma_k^2 I_{d_v})$.

Define the sequence of representation parameter and the predictor parameter for $K$ iterations as $(U,V)^{[K]} = (U^{[K]},V^{[K]}) = ((U_1, V_1), (U_2,V_2),...,(U_K,V_K))$. The sequence of samplings for each dataset are defined as $T_{B}^{[K]} = (T_B^1,...,T_B^K)$, $S_{B_{1:N}}^{[K]}= (S_{B_{1:N}}^1,..., S_{B_{1:N}}^K)$ and $T_{X_B}^{[K]}=(T_{X_B}^1,...,T_{X_B}^K)$. The final output of the joint optimization algorithm $(U,V)= f((U,V)^{[K]})$. 

\[
\begin{aligned}
    T \rightarrow  T_B^{[K]}& \\
              \downarrow\\
    S_{1:N}\rightarrow S_{B_{1:N}}^{[K]} \rightarrow (U,& V)^{[K]} \rightarrow (U,V)\\
    \uparrow\\
    T_X^\prime \rightarrow  T_{X_B}^{[K]}&
\end{aligned}
\]

Given specific $\vect{\alpha}$, let us apply the data processing inequality and the chain rule, then we have
\[
I(U,V;S^{\vect{\alpha}}, T) = I(U,V;S_{1:N}, T) \leq I((U^{[K]},V^{[K]}); S_{B_{1:N}}^{[K]}, T_B^{[K]}) = \sum_{k=1}^K I(U_k,V_k;S_{B_{1:N}}^{[K]}, T_B^{[K]} | U^{[k-1]}, V^{[k-1]})
\]
Moreover, since the sampling strategy is agnostic to the previous iterates of the parameters and previous samplings, so we can obtain: 
\[
\begin{aligned}
I(U_k,V_k;S_{B_{1:N}}^{[K]}, T_B^{[K]} | U^{[k-1]}, V^{[k-1]}) &=
I(U_k,V_k;S_{B_{1:N}}^{k}, T_B^{k} | U^{k-1}, V^{k-1})\\
&=H(U_k,V_k|U_{k-1}, V_{k-1}) - H(U_k,V_k|U_{k-1},V_{k-1}, S_{B_{1:N}}^{k}, T_B^{k})
\end{aligned}
\]

For any $(U_{k-1}, V_{k-1})=(u_{k-1}, v_{k-1})$, since the joint entropy of a set of variables is less than or equal to the sum of the individual entropies of the variables in the set, we have

\[
\begin{aligned}
H(U_k,V_k|U_{k-1}=u_{k-1},V_{k-1}=v_{k-1}) \leq H(-\eta_k^u G_u^k + \xi_u^k) + H(- \eta_k^v G_v^k + \xi_v^k))
\end{aligned}
\]

Apply Lemma~\ref{lemma_entropy}, we can obtain
\[
\begin{aligned}
H(U_k,V_k|U_{k-1}=u_{k-1},V_{k-1}=v_{k-1},& S_{B_{1:N}}^{k}, T_B^{k})\\
&= H(-\eta_k^u G_u(u_{k-1}, v_{k-1}, T_{X_B}^k) + \xi_u^k , -\eta_k^v  G_v(u_{k-1}, v_{k-1}, T_{X_B}^k) + \xi_v^k)\\
&\geq H(-\eta_k^u G_u(u_{k-1}, v_{k-1}, T_{X_B}^k) -\eta_k^v  G_v(u_{k-1}, v_{k-1}, T_{X_B}^k) + \xi_u^k  + \xi_v^k)\\
&\geq H(\xi_u^k) + H(\xi_v^k) = \frac{d_u}{2}\log(2\pi e \sigma_k^2) + \frac{d_v}{2}\log(2\pi e \sigma_k^2) \,.
\end{aligned}
\]
The first inequality is obtained because the entropy of a function of random variables is smaller than the entropy of the random variables (See proof in Lemma~\ref{lemma_entropy}). The second inequality comes from the independence between the gradients and the injected noise.

Moreover, we can obtain $\EE(\|-\eta_k^u G_u^k + \xi_u^k\|_2^2) = \EE(\|\eta_u^k G_u^k\|_2^2 + \|\xi_u^k\|_2^2) = (\eta_u^k)^2 \EE\|G_u^k\|_2^2 + d_u\sigma_k^2$ and  $\EE(\|-\eta_k^v G_v^k + \xi_v^k\|_2^2) =\EE(\|\eta_v^k G_v^k\|_2^2 + \|\xi_v^k\|_2^2)  = (\eta_v^k)^2 \EE\|G_v^k\|_2^2 + d_v\sigma_k^2$ with the independence. Since Gaussian distribution has the largest entropy among the variables with the same second order moment, so we can further get:
\[
\begin{aligned}
H(U_k,V_k|U_{k-1}=u_{k-1}, V_{k-1}=v_{k-1}) - &H(U_k,V_k|U_{k-1}=u_{k-1},V_{k-1}=v_{k-1}, S_{B_{1:N}}^{K}, T_B^{k})\\
&\leq H(-\eta_k^u G_u^k + \xi_u^k) + H(- \eta_k^v G_v^k + \xi_v^k)) - (H(\xi_u^k) + H(\xi_v^k))\\
&\leq \frac{d_u}{2}\log(2\pi e\frac{(\eta_u^k)^2\EE\|G_u^k\|_2^2 + d_u\sigma_k^2}{d_u}) + \frac{d_v}{2}\log(2\pi e\frac{(\eta_v^k)^2\EE\|G_v^k\|_2^2 + d_v\sigma_k^2}{d_v})\\
&\quad\quad - \frac{d_u}{2}\log(2\pi e \sigma_k^2) - \frac{d_v}{2}\log(2\pi e \sigma_k^2) \\
&\leq \frac{(\eta_u^k)^2\EE\|G_u^k\|_2^2 + (\eta_v^k)^2\EE\|G_v^k\|^2}{2\sigma_k^2}
\end{aligned}
\]
The above bound holds for all $u_{k-1},v_{k-1}$, thus we can integrate the bound to conclude that $$H(U_k,V_k|U_{k-1}, V_{k-1}) - H(U_k,V_k|U_{k-1},V_{k-1}, S_{B_{1:N}}^{K}, T_B^{k}) \leq \frac{(\eta_u^k)^2\EE\|G_u^k\|_2^2 + (\eta_v^k)^2\EE\|G_v^k\|^2}{2\sigma_k^2}\,,$$
So we get $I(U,V;S^{\vect{\alpha}}, T) \leq \sum_{k=1}^K \frac{(\eta_u^k)^2\EE\|G_u^k\|_2^2 + (\eta_v^k)^2\EE\|G_v^k\|^2}{2\sigma_k^2} = \delta_u + \delta_v$

Similarly, we can obtain the following results with the same proof process:
\[
\begin{aligned}
I(U;S^{\vect{\alpha}}, T) \leq \sum_{k=1}^K \frac{(\eta_u^k)^2\EE\|G_u^k\|_2^2 }{2\sigma_k^2} = \delta_u
\end{aligned}
\]
\[
I(U,V;S^{\vect{\alpha}}, T_X^\prime)\leq \sum_{k=1}^K \frac{(\eta_u^k)^2\EE\|G_u^k\|_2^2 + (\eta_v^k)^2\EE\|G_v^k\|^2}{2\sigma_k^2} = \delta_u + \delta_v
\]
Place the above terms into the inequality presented at the beginning of this section, we can conclude the proof.
\end{proof}
\newpage
\section{EXPERIMENTAL DETAILS}\label{sec:exp_details}
We provide the experimental details and some additional results in this section. The pseudo code \footnote{The code is released at \url{https://github.com/livreQ/IMDA}} for the proposed IMDA algorithm is presented in Sec.~\ref{sec:imda}.

\subsection{Amazon Review}
As presented in the main paper, the original dataset is pre-processed to $5000$-dimension bag-of words features following \citet{chen2012marginalized}. And the target shift data is created by randomly dropping $50\%$ negative reviews. 

\paragraph{Model Structure} 
\textit{Representation learner}: $[5000, 1000, 500, 100]$ MLP net using $0.7$ dropout rate with Relu activation added after each hidden layer and finally output a $100$-dimension feature representation. 

\textit{Predictor and duplicate predictor}: $[100, 2]$ linear transformation followed by a log softmax layer transforming the $100$-dimension feature to $2$-class log probabilities.

\textit{Loss function}: we choose the "negative log-likelihood loss" as the loss function.

\paragraph{Computing Resources} The experiments were run on a server with 6 CPUs and 1 GPU of 32GB memory.

\paragraph{Experimental Setting} The Amazon review dataset contains 6465 samples for "books", 5586 samples for "dvd", 7681 samples for "electronics" and 7945 samples for "kitchen". We randomly sample 2000 examples for each domain. In unsupervised MDA, the 2000 target samples without labels are applied in the training phase, and the rest samples are used as the test set. For supervised MDA with few target labels, we randomly sample $10\%$ examples from the 2000 target samples using as the train data, then the rest $90\%$ are used as the test set. 

We provide the main hyper-parameters for supervised and unsupervised settings, which are chosen based on cross-validation. Other hyper-parameters are set as the default value provided in the code. 

In the supervised MDA, we set l2\_scale$=0.5$ (the $C_1$ constant in the convex optimization objective for $\vect{\alpha}$), and the learning rate of a Adadelta optimizer is $\eta=0.5$. The penalty for the $\mathbf{W}_1$ distance is set to W1\_sup\_coef$=0.01$. Finally, the network is trained for $40$ epochs with a mini-batch size of $20$.

In the unsupervised MDA, as presented in the main paper, the empirical risk of the pseudo target distribution is:
$\hat{R}_{\Tcal_{u,v}}(u,v, v^\prime) = \frac{1}{m_t^\prime} \sum_{j=1}^{m_t^\prime} \ell(h(v^\prime, g(u,X_j^t)), h(v, g(u, X_j^t))$.
There does not exist a direct realization of the above loss. So we implement it with $\ell(h(v^\prime, g(U,X_j^t)), \hat{Y}_j^t) + \ell(h(V, g(U, X_j^t))),\hat{Y}^{\prime t}_j)$, where the $\hat{Y}_j^t$ and $\hat{Y}^{\prime t}_j$ are the predicted label of corresponding predictor.
Then we apply two different penalty weights W1\_discri\_coef1 and W1\_discri\_coef2 for the two losses respectively.
We set l2\_scale$=1$, learning rate $\eta=0.8$, W1\_discri\_coef1$=0.06$, $C_0$=W1\_discri\_coef2$=1.2$. Finally, the network is trained for $50$ epochs and the mini-batch size is $20$. 

For reproducing the results of WADN \citep{shui2021aggregating} and DARN \citep{wen2020domain}, we used the code repos \url{https://github.com/cjshui/WADN} and \url{https://github.com/junfengwen/DARN}. The hyper-parameters are set as provided in the respective paper.

\subsection{Digits}

\paragraph{Model Structure} \textit{Representation learner} is consisted of 3 stacked modules : 'conv':[3, 3, 64], 'relu', 'maxpooling:'[2,2,0]; 'conv':[3, 3, 128], 'relu', 'maxpooling:'[2,2,0];'conv':[3, 3, 256], 'relu', 'maxpooling:'[2,2,0]. 

\textit{Predictor and duplicate predictor}: $[2304, 512, 100, 10]$ MLP net using "Relu" activation after each hidden layer is followed by a log softmax layer. The predictors transform the $2304$-dimension feature to $10$-class log probabilities. There is no dropout.

\textit{Loss function}: we choose the "negative log-likelihood loss" as the loss function.

\paragraph{Computing Resources} The experiments were run on a server with 6 CPUs and 1 GPU of 32GB memory.

\paragraph{Experimental Setting}
Digits dataset has four domains: "MNIST," "USPS," "SVHN," and "SYNTH." "MNIST" and "USPS" are handwriting datasets that are very close. "SYNTH" is transformed from "SVHN," so they are also more similar to each other. Each domain has different train-test split when downloaded. The respective training sample sizes are 60000, 7219, 73257, 479400, and the respective test sample sizes are 10000, 2017, 26032, 9553. Following the same procedure of \citet{shui2021aggregating}, we randomly select 7000 samples for each domain, then create a target shifted distribution for each source, making each source contain 5300 samples, and the target domain contains 7000 examples. 

In unsupervised MDA, the 7000 target samples without labels are applied in the training phase, and the test set is an unseen portion that contains 2000 samples. For supervised MDA with few target labels, we randomly sample $10\%$ examples from the 7000 target samples using as the train data, then the rest $90\%$ are used as the test set. 

We provide the main hyper-parameters for supervised and unsupervised settings, which are chosen based on cross-validation. Other hyper-parameters are set as the default value provided in the code. 


In the supervised MDA, we set l2\_scale$=1.5, 1$, and the learning rate of a Adadelta optimizer is $\eta=0.5, 0.2$, The penalty for the $\mathbf{W}_1$ distance is set to W1\_sup\_coef$=0.01, 0.03$. Finally, the network is trained for $70$ epochs with a mini-batch size of $128$. In the unsupervised MDA, we set l2\_scale$=1.5$, learning rate $\eta=0.4$, W1\_discri\_coef1$=0.002$, $C_0$=W1\_discri\_coef2$=0.9$. Finally, the network is trained for $15$ epochs and the mini-batch size is $128$. 

\subsection{Gradient Penalty}
We implemented two types of gradient penalties in the code. The first is to ensure the Lipschitzness of $\tilde{f}$ for the Wasserstein distances. We adopt the method proposed by \citet{gulrajani2017improved}. For example, with $(\tilde{X},Y)\sim\Tcal_u$ and  $(\tilde{X}^\prime, Y^\prime)\sim\Scal^{\vect{\alpha}}_u$, we generate the interpolated feature $\tilde{X}_{int} = \lambda \tilde{X} + (1-\lambda)\tilde{X}^\prime, \lambda\sim \text{Unif}[0,1]$. The gradient penalty is implemented by adding the regularization term $\|\nabla_{\tilde{X}_{int}} h(v,\tilde{X}_{int}) \|_2^2$. The second gradient penalty comes from the bound in Theorem 6.1. We consider the gradient norm in each batch as a regularization term. Note this gradient is w.r.t the model parameters, not the feature or input.

\newpage
\section{ADDITIONAL RESULTS}\label{sec:add_exp}

\subsection{Amazon Review}
Now we present some additional results. At first, we illustrate the ablation study w.r.t different drop rates on Amazon dataset for each domain in unsupervised MDA. In Fig.~\ref{fig:ablation}, we can see the proposed IMDA algorithm consistently outperforms the previous state-of-art on every single domain. 

We further provide additional t\_SNE visualization of amazon data in Fig.~\ref{fig:tsne}, where an illustration of target shift with drop rate of $50\%$ is also included. From the figure, we observe that the geometry information of the data is retained using the Wasserstein distance, and the distribution shift is decreased after adaptation.

\begin{figure}[H]
    \centering
    \includegraphics[width=0.4\textwidth]{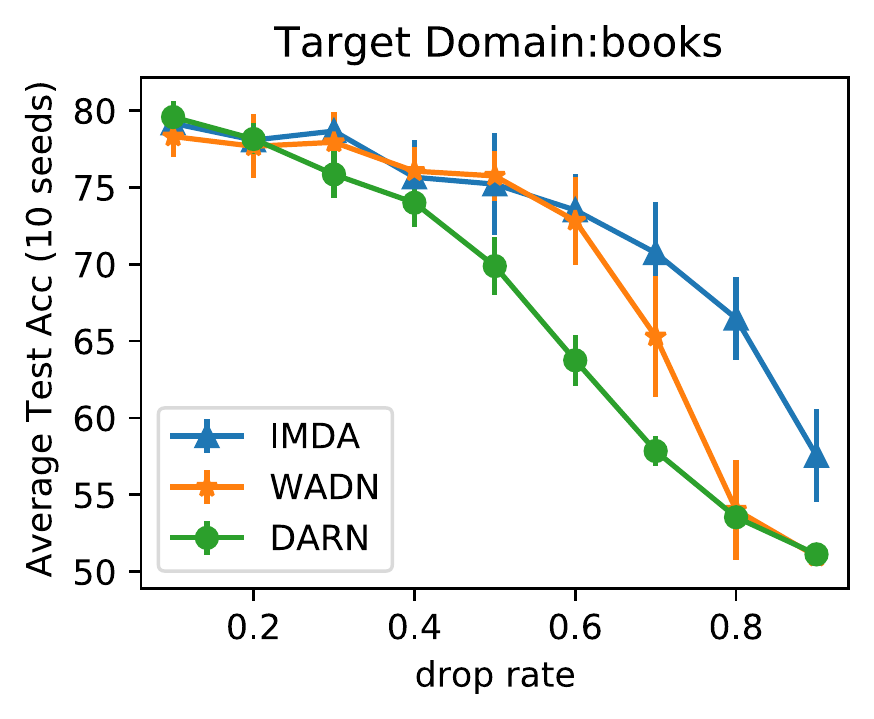}
    \includegraphics[width=0.4\textwidth]{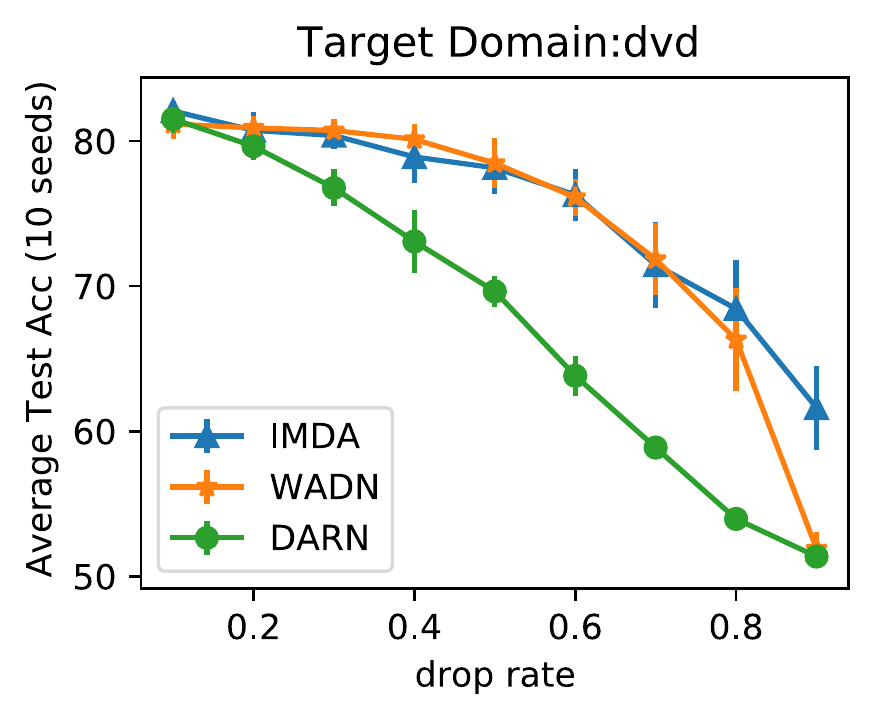}
    \includegraphics[width=0.4\textwidth]{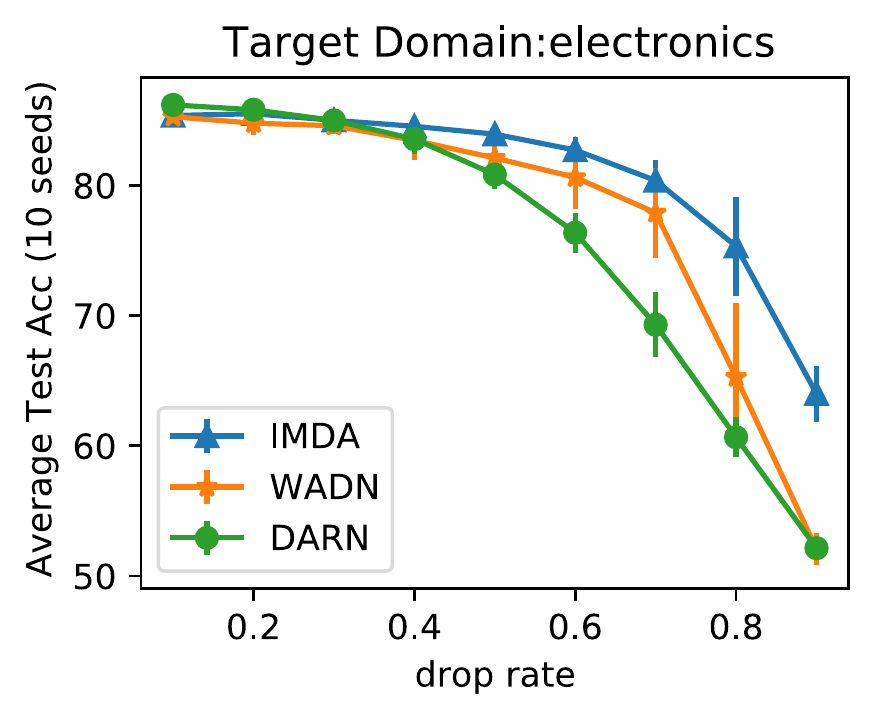}
    \includegraphics[width=0.4\textwidth]{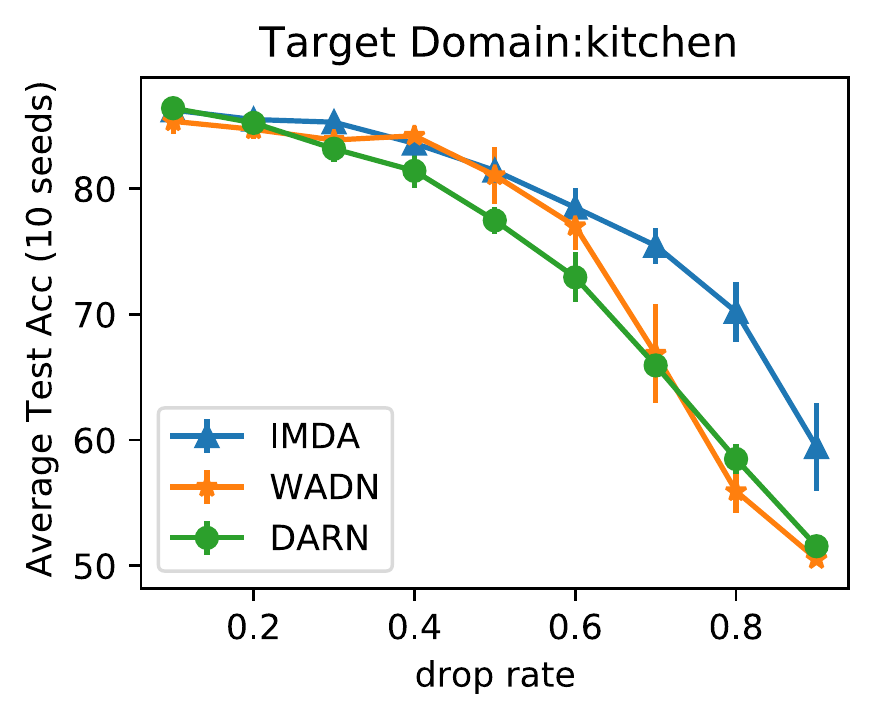}
    \caption{Test Accuracy for Each Domain w.r.t Different Levels of Target Shift on Amazon Review Dataset for Unsupervised MDA}
     \label{fig:ablation}
\end{figure}

\subsection{Digits}
The Amazon review dataset's four domains are homogeneous tasks with similar difficulty. Otherwise, the four digits domains are heterogeneous, where the two handwriting digit recognition tasks are much easier than the two home number recognition tasks. To learn more reliable relations, we set uniform weight at the beginning of the training process (5 epochs) to make the complex tasks sufficiently learned, avoiding an unfair low weight due to the poor results.

From the dataset description, we know MNIST and USPS are handwriting digits and are more similar to each other than other datasets. Moreover, SVHN and SYNTH are more correlated since SYNTH is obtained by adding transformation on SVHN. We observe the apparent similarity in Fig.~\ref{fig:digits} (a). The algorithm gives high weight to USPS when learning MNIST and high weight to SVHN when learning SYNTH, and vice versa.
Fig.~\ref{fig:digits} illustrates the minimax optimization process for the Wasserstein distance. In unsupervised MDA, we do not directly optimize the combined empirical risk $\hat{R}_{\Scal^{\vect{\alpha}}}(u,v)$, the evolution of its estimation continuously decreases, which shows the effectiveness of minimizing $\mathbf{W}_1(\tilde{\Tcal}_{u,v}, \tilde{\Scal}_u)$. 

\begin{figure}[H]
\centering
\includegraphics[width=0.9\linewidth]{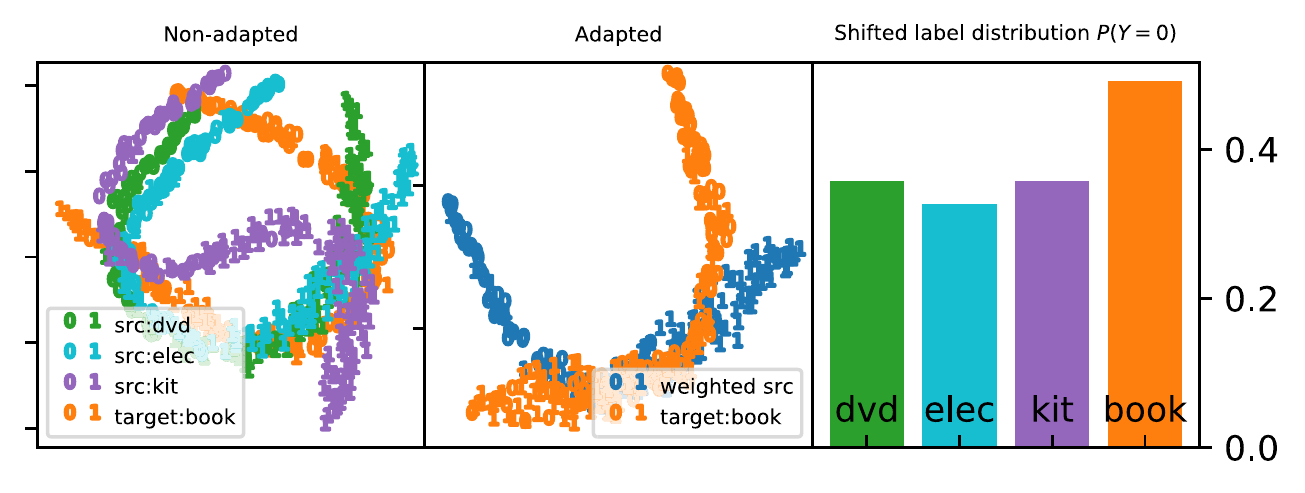}
\caption{T\_SNE \& Target Shift on Amazon Review Dataset}
\label{fig:tsne}
\end{figure}

We also test the single-source DA method DeepJDOT \citep{damodaran2018deepjdot} (merging the sources as one) and MOST \citep{nguyen2021most} on the target shifted digits data. The result for the target domain MNIST is illustrated in Tab.~\ref{tab:add}. We do not consider adding a more detailed comparison of these two methods for the following reasons. First, we have already compared other similar single-source DA approaches that using the merged source data. Second, the original MOST algorithm consist of a mixture of several optimization objectives, but the theoretical result is only related to one objective. Therefore, adjusting hyper-parameters for these combinations is meaningless for this paper's theoretical nature. 

\begin{table}[H]
\caption{Accuracy(\%) on Digits Dataset with a Drop Rate of 50\%}
\centering
\begin{tabular}{|c|c|}
Method & target: MNIST \\
\hline
MOST &88.23 \\  
DeepJDOT &87.5\\  
IMDA& 89.26
\end{tabular}
\label{tab:add}
\end{table} 

\begin{figure}[!ht]
    \centering
    \subfloat[Evolution of Domains Weights $\vect{\alpha}$ ]{
    \includegraphics[width=0.52\textwidth]{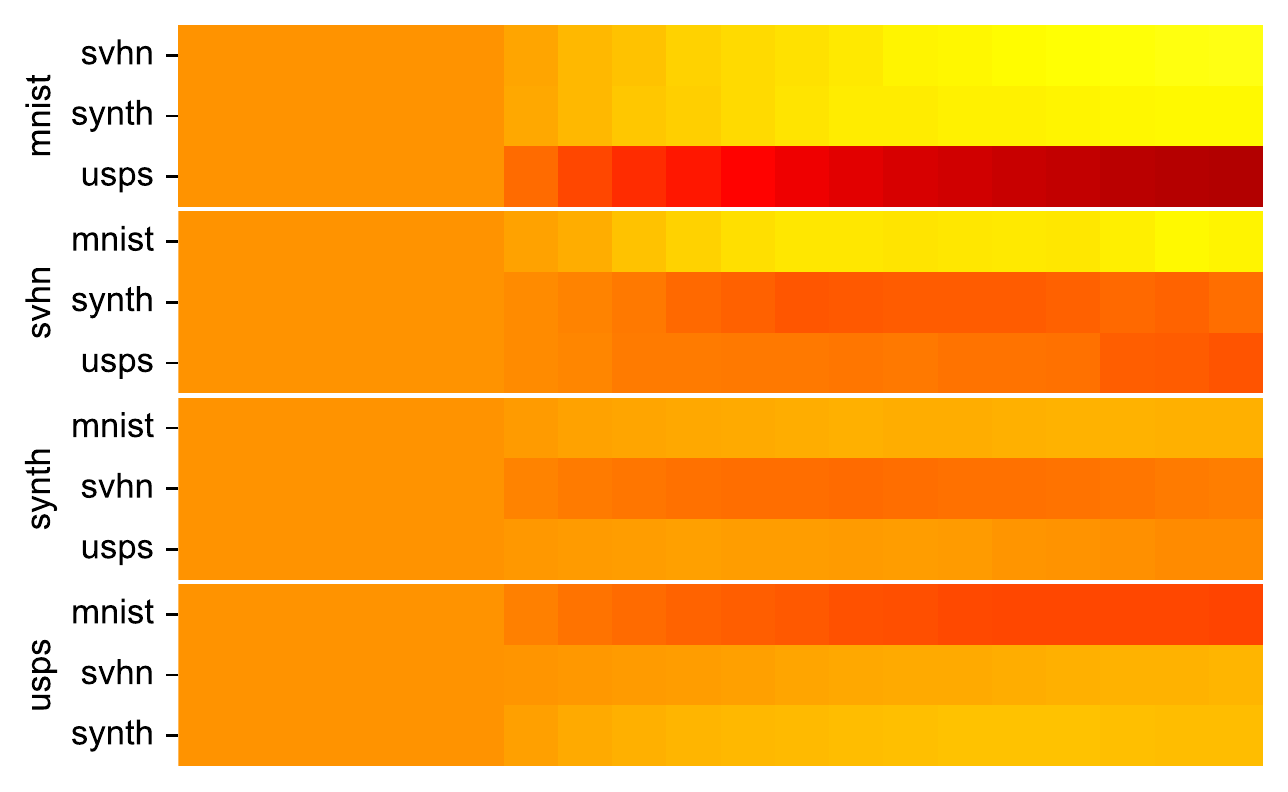}
    }
    \subfloat[Evolution of $\hat{R}_{\Scal^{\vect{\alpha}}}(u,v)$ and $\hat{\mathbf{W}}_1(\tilde{\Tcal}_{u,v}, \tilde{\Scal}_u)$]{
    \includegraphics[width=0.4\textwidth]{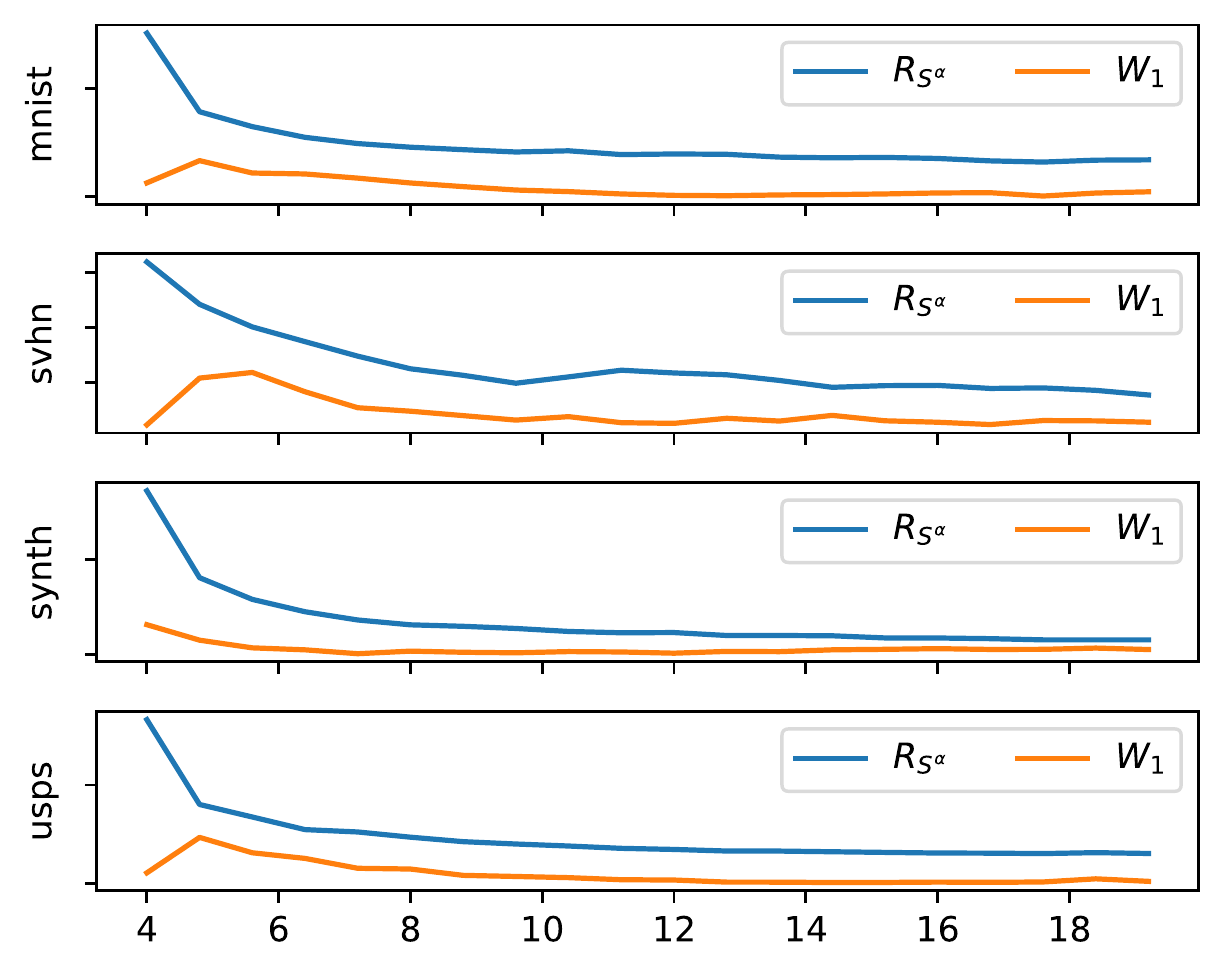}
    }
    \caption{Visualization of the Evolution of Different Terms w.r.t Training Epochs on Digits Dataset for Unsupervised MDA with a Drop Rate of 0.5}
     \label{fig:digits}
\end{figure}

\newpage
\section{ALGORITHM}\label{sec:imda}
\begin{algorithm}[H]
\SetAlgoLined
\textbf{Input:} Labeled source samples $S_{1:N}$,
                  few labeled target samples $T$,
                  unlabeled target samples $T_X^\prime$,
                  weights $\epsilon$, $\tau$, learning rate $\eta$, variance $\sigma$, batch size $B$, constants $C_0, C_1$, moving average weight $0<C<1$\;
\textbf{Output:} Representation parameter $U$, predictor parameter $V$, duplicate predictor parameter $v^\prime$, domain weights $\vect{\alpha}$\;
\For{$t \gets 1$ to $T$}{
    \eIf{$t == 1$}{
     Initialize domain weights $\vect{\alpha}=\{1/N,...,1/N\}$, randomly initialize $U_0,V_0,v^\prime_0$\;
     $\delta_u=0,\delta_v=0$\;
    }
    {
      Initialize $U_0,V_0,v^\prime_0$ with last epoch outputs\;
    }
    \For{$k \gets 1$ to $K$}{
      Sample batch data $T_B^{k},S_{B_{1:N}}^k,T_{X_B}^k$ from the corresponding datasets\;
      Evaluate gradients:
      \[
      \begin{aligned}
            G_u^k =& \tau(1-\epsilon)\nabla_U\hat{R}_{T_{B}^k}(U_{k-1},V_{k-1}) \\
            & + \tau\epsilon \nabla_U \hat{R}_{T_{B}^k}(U,v^\prime_{k-1}) + (1-\tau) \nabla_U \hat{R}_{T_{X_B}^k}(U_{k-1},V_{k-1}, v^\prime_{k-1}) \\
            & + \tau\epsilon \nabla_U \hat{R}_{S_{B_{1:N}}^k}^{\vect{\alpha}}(U_{k-1},V_{k-1}) -(1-\tau + \epsilon\tau)\nabla_U \hat{R}_{S_{B_{1:N}}^k}^{\vect{\alpha}}(U_{k-1},v^\prime_{k-1})\\
          G_v^k = &\tau(1-\epsilon)\nabla_V\hat{R}_{T_{B}^k}(U_{k-1},V_{k-1}) \\
          & + (1-\tau) \nabla_V \hat{R}_{T_{X_B}^k}(U_{k-1},V_{k-1}, v^\prime_{k-1}) + \tau\epsilon \nabla_V \hat{R}_{S_{B_{1:N}}^k}^{\vect{\alpha}}(U_{k-1},V_{k-1})\\
          G_{v^{\prime}}^k =& \tau\epsilon [\nabla_{v^\prime}  \hat{R}_{T_{B}^k}(U,v^\prime_{k-1}) -  \nabla_{v^\prime} \hat{R}_{S_{B_{1:N}}^k}(U,v^\prime_{k-1})] \\
          &+ (1 - \tau)[ \nabla_{v^\prime}\hat{R}_{T_{X_B}^k}(U,V, v^\prime_{k-1}) - \nabla_{v^\prime} \hat{R}_{S_{B_{1:N}}^k}(U,v^\prime_{k-1})]
     \end{aligned}
     \]
     Update:
     \[\begin{aligned}
     U_k &= U_{k-1} - \eta G_u^k + \xi_k^u,
     V_{k} &= V_{k-1} - \eta G_v^k + \xi_k^v,
     v^\prime_k &= v^\prime_{k-1} + \eta G_{v^\prime}^k\,;
     \end{aligned}
     \]
     $\delta_{u} = \delta_{u} + \frac{\eta^2\EE\|G_u^k\|_2^2 }{2\sigma^2}$, $\delta_v = \delta_v +  \frac{\eta^2\EE\|G_v^k\|_2^2}{2\sigma^2}$\; 
    }
    $U=U_K,V=V_K,v^\prime=v^\prime_K$\;
    Solve
    \[\begin{aligned}
     \vect{\alpha}^\prime = \min_{\vect{\alpha}} \space &\left( (\epsilon\tau + C_0(1-\tau)) \hat{R}_{S_{B_{1:N}}^k}^{\vect{\alpha}}(U,V) - (\epsilon\tau + 1- \tau) \hat{R}_{S_{B_{1:N}}^k}^{\vect{\alpha}}(U,v^\prime) \right. \\
     &\left.\quad\quad + C_1((1-\tau + \tau\epsilon)\sqrt{\delta_u + \delta_v} + \tau\epsilon \sqrt{\delta_u}) R(\vect{\alpha}) \right)\, ,\\
    &R(\vect{\alpha}) = \sqrt{\sum_{i=1}^N \frac{\alpha_i^2}{m_i}},\ \text{s.t.} \forall i\in[N], \alpha_i\geq 0, \sum_{i=1}^N \alpha_i = 1\,,
\end{aligned}
\]
Update $\vect{\alpha} = C\vect{\alpha} + (1-C)\vect{\alpha}^\prime$\;
}
\caption{Information-Theoretic Multi-Source Domain Adaptation (IMDA) }
\label{algo:imda}
\end{algorithm}

\vfill

\end{document}

%% file: math_commands.tex
\usepackage{amsmath,amsfonts,bm}

\def\eqref#1{equation~\ref{#1}}

\def\1{\bm{1}}

\DeclareMathAlphabet{\mathsfit}{\encodingdefault}{\sfdefault}{m}{sl}
\SetMathAlphabet{\mathsfit}{bold}{\encodingdefault}{\sfdefault}{bx}{n}

\newcommand{\KL}{D_{\mathrm{KL}}}

\DeclareMathOperator*{\argmin}{arg\,min}

\ifdefined\eqdef
\else
    \newcommand{\eqdef}
    {
        \overset{{\rm \mbox{\tiny def}}}{=}
    }
\fi

\newtheorem{theorem}{Theorem}[section]

\newtheorem{lemma}[theorem]{Lemma}
\newtheorem{definition}{Definition}[section]
\newtheorem{assumption}{Assumption}
\usepackage{mathtools}
\newtheorem*{theorem*}{Theorem}

\usepackage{mathtools}

\newcommand{\vect}[1]{\boldsymbol{#1}}
\newcommand{\ie}{{\em i.e.\/}}
\newcommand{\eg}{{\em e.g.\/}}

\newcommand{\Acal}{{\mathcal A}}

\newcommand{\Vcal}{{\mathcal V}}
\newcommand{\Ucal}{{\mathcal U}}
\newcommand{\Hcal}{{\mathcal H}}
\newcommand{\Fcal}{{\mathcal F}}

\newcommand{\Ocal}{{\mathcal O}}

\newcommand{\Xcal}{{\mathcal X}}
\newcommand{\Ycal}{{\mathcal Y}}

\newcommand{\Scal}{{\mathcal S}}
\newcommand{\Tcal}{{\mathcal T}}
\newcommand{\Zcal}{{\mathcal Z}}

\newcommand{\EE}{{\mathbb{E}}}

\newcommand{\Reals}{{\mathbb{R}}}